\documentclass{article}

\usepackage{microtype}
\usepackage{graphicx}
\usepackage{subcaption}
\usepackage{booktabs} 

\usepackage{hyperref}

\usepackage[preprint]{icml2026}

\usepackage{amsmath}
\usepackage{amssymb}
\usepackage{mathtools}
\usepackage{amsthm}

\usepackage{graphicx}
\usepackage{algorithm}
\usepackage{algorithmic}
\usepackage{amsfonts}
\usepackage{multirow}
\usepackage{makecell}
\usepackage{caption}
\usepackage{enumitem}

\usepackage[capitalize,noabbrev]{cleveref}

\theoremstyle{plain}
\newtheorem{theorem}{Theorem}[section]
\newtheorem{proposition}[theorem]{Proposition}
\newtheorem{lemma}[theorem]{Lemma}

\theoremstyle{definition}

\theoremstyle{remark}

\usepackage[textsize=tiny]{todonotes}

\icmltitlerunning{Submission and Formatting Instructions for ICML 2026}

\begin{document}

\twocolumn[
  \icmltitle{BiSpikCLM: A Spiking Language Model integrating \\ Softmax-Free Spiking Attention and Spike-Aware Alignment Distillation
    }



  \icmlsetsymbol{equal}{*}

    \begin{icmlauthorlist}
        \icmlauthor{Sihang Guo}{hit}
        \icmlauthor{Chenlin Zhou}{pku,pcl}
        \icmlauthor{Jiaqi Wang}{hit,pcl}
        \icmlauthor{Kehai Chen}{hit}
        \icmlauthor{Qingyan Meng}{pcl}
        \icmlauthor{Zhengyu Ma}{pcl}
    \end{icmlauthorlist}
    
    \icmlaffiliation{pku}{School of Electronic and Computer Engineering, Shenzhen Graduate School, Peking University, Shenzhen, China}
    \icmlaffiliation{pcl}{Peng Cheng Laboratory, Shenzhen, China}
    \icmlaffiliation{hit}{Harbin Institute of Technology, Shenzhen, China}
    
    \icmlcorrespondingauthor{Zhengyu Ma}{mazhy@pcl.ac.cn}

  \icmlkeywords{Machine Learning, ICML}

  \vskip 0.3in
]



\printAffiliationsAndNotice{}  

\begin{abstract}
Spiking Neural Networks (SNNs) offer promising energy-efficient alternatives to large language models (LLMs) due to their event-driven nature and ultra-low power consumption. However, to preserve capacity, most existing spiking LLMs still incur intensive floating-point matrix multiplication (MatMul) and nonlinearities, or training difficulties arising from the complex spatiotemporal dynamics. To address these challenges, we propose BiSpikCLM, the first fully binary spiking MatMul-free causal language model. BiSpikCLM introduces Softmax-Free Spiking Attention (SFSA), eliminating softmax and floating-point operations in autoregressive language modeling. For efficient training, we introduce Spike-Aware Alignment Distillation (SpAD), which aligns ANN teacher and SNN student across embeddings, attention maps, intermediate features, and output logits. SpAD framework allows BiSpikCLM to reach comparable performance to ANN counterparts using substantially fewer training tokens (e.g., only 5.6\% of the tokens for the 1.3B model). As a result, BiSpikCLM achieves competitive performance at only 4.16\%–5.87\% of the computational cost on natural language generation tasks. Our results highlight the feasibility and effectiveness of fully binary spike-driven LLMs and establish the distillation as a promising pathway for brain-inspired spiking NLP.
\end{abstract}

\begin{figure*}[t]
\centering
\includegraphics[width=0.95\linewidth]{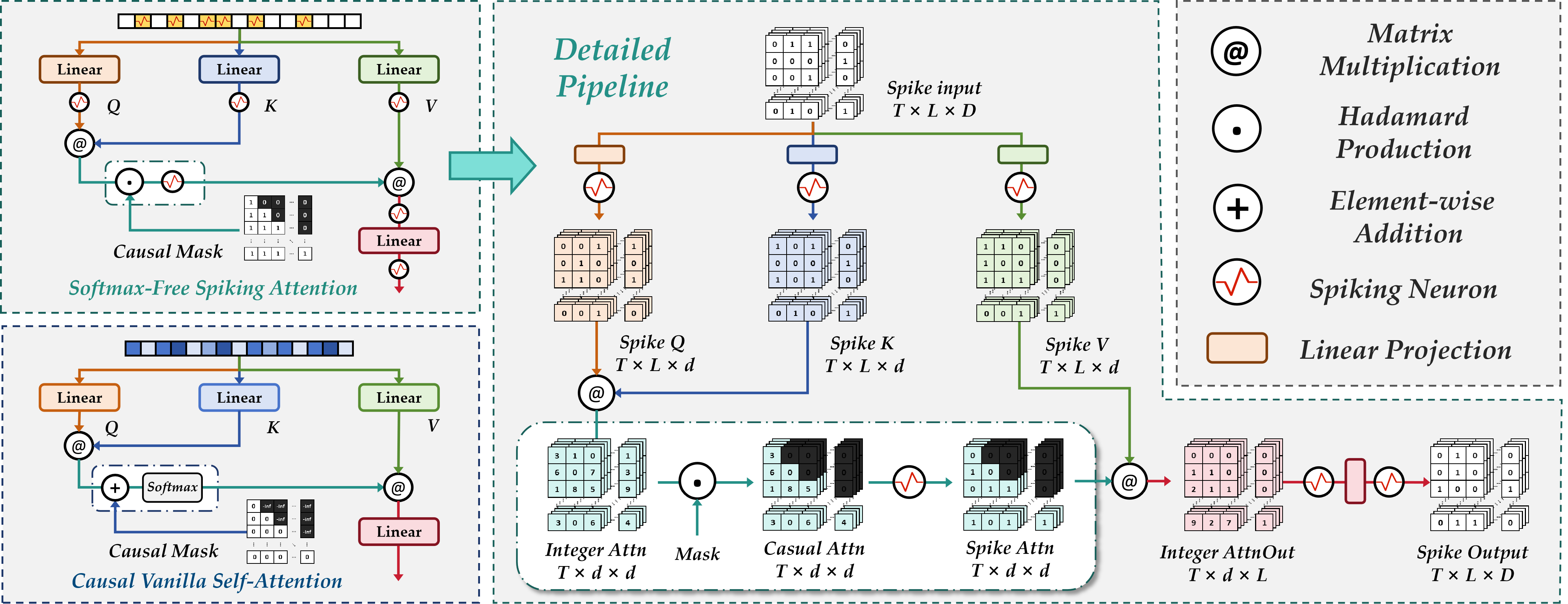}
\caption{Overview of \textbf{Softmax-Free Spiking Attention (SFSA)}. 
Left: Comparison between Vanilla Causal Self-Attention (CSA) (bottom) and SFSA (top). CSA uses softmax and additive masks, while SFSA employs spike-based activation and binary causal masking. Right: Detailed SFSA pipeline, showing spike-form Q, K, V computation, masked integer attention, spiking activation, and spike-based output, enabling fully discrete and energy-efficient attention modeling.}
\label{fig:sfsa}
\end{figure*}

\section{Introduction}

Large Language Models (LLMs) have demonstrated remarkable capabilities in natural language processing, powering a wide range of applications from conversational agents to code generation~\citep{brown2020language,achiam2023gpt}. However, these models typically require extensive computational resources and energy consumption during both training and inference. For example, GPT-3 was trained with 175 billion parameters using hundreds of petaflop/s-days of compute~\citep{brown2020language}. In addition, inference also incurs substantial energy costs, as serving a single query can involve billions of operations and significant GPU utilization~\citep{strubell2020energy,schwartz2020green}, raising concerns about their scalability and environmental impact~\citep{strubell2020energy}.

Unlike ANN-based LLMs, the human brain achieves superior intelligence with far lower energy consumption, operating on just 20 watts to power approximately 86 billion neurons~\citep{izhikevich2003simple,gerstner2014neuronal}. Inspired by this energy-efficiency, Spiking Neural Networks (SNNs)~\citep{maass1997networks,gerstner2014neuronal} communicate through binary spike events, enabling event-driven and low-power computation~\citep{yin2021accurate,schuman2022opportunities}, particularly on neuromorphic hardware, making SNNs a promising alternative to traditional ANNs. 

While recent efforts have shown promising results of SNNs in computer vision tasks~\citep{zhou2024qkformer,li2024spikeformer,luo2024integer}, extending SNNs to natural language processing (NLP), especially LLMs, remains largely underexplored. A central challenge is the design of spiking attention mechanisms. In contrast to vision models, where representations are often bidirectional and spatially local, autoregressive LLMs require \textbf{causal attention} to ensure that each token prediction depends only on its preceding context. However, conventional causal attention relies on floating-point matrix multiplications and the softmax operation, both of which are computationally intensive and fundamentally incompatible with spike-based processing. Existing attempts either retain these components~\citep{zhu2023spikegpt} or introduce multi-threshold neurons and integer activations~\citep{xing2024spikellm}, which still incur substantial floating-point overhead. Designing a spike-driven causal attention mechanism is therefore critical: it must eliminate softmax while preserving the autoregressive representational capacity of binary spike trains. This challenge directly motivates our \textbf{Softmax-Free Spiking Attention (SFSA)}, which enables spike-based language modeling for spiking LLMs.

Moreover, training spiking LLMs introduces additional difficulties beyond those in vision tasks. The inherent temporal dynamics of SNNs already leads to complex computational graphs and high computational cost during backpropagation. Scaling up the architecture further exacerbates this, making full end-to-end training inefficient or even infeasible. Consequently, prior works mostly resort to ANN-to-SNN conversions~\citep{xing2024spikellm,schmidgall2024brain}. However, such methods typically require large time steps to approximate ANN activations, resulting in high inference cost. Integer-based conversions further scale the operations by $T \times N$, which compromises the potential energy benefits of event-driven spiking computation.

To address these challenges, we propose BiSpikCLM, a binary spiking causal language model built on two key components: a spike-driven attention mechanism (SFSA), schematically depicted in Figure~\ref{fig:sfsa}, and a spike-aware knowledge distillation scheme (SpAD), presented in Figure~\ref{fig:distill}. Overall, our contributions can be summarized as follows:

\begin{itemize}[leftmargin=*]
\item We propose \textbf{BiSpikCLM}, the first binary spiking MatMul-free causal language model equipped with a fully spike-driven attention mechanism. Our \textbf{SFSA} replaces the causal self-attention, which relies on floating-point operations and softmax, enabling efficient autoregressive language modeling with binary spikes. The overall design follows the OPT-family architecture~\citep{zhang2022opt}, adapted to the spiking domain.

\item We introduce \textbf{SpAD}, a novel training framework that enables BiSpikCLM to be directly trained from random initialization. SpAD distills hierarchical knowledge covering embeddings, attention maps, intermediate features, and output logits from the teacher model, thereby accelerating convergence and reducing the amount of training data required for large-scale spiking LLMs.

\item With only 10B training tokens, significantly fewer than the 180B tokens used to train OPT-1.3B, our SpAD framework enables BiSpikCLM-1.3B to achieve 42.19\% zero-shot accuracy on common reasoning benchmarks using 4 time steps, approaching the 49.73\% of OPT-1.3B, while consuming just 10.6\% of the energy per inference. Remarkably, even at 2 time steps, the model maintains 41.33\% accuracy with only 5.88\% of the energy cost.
\end{itemize}

\begin{table*}[t]
\caption{
Comparison of spiking language models.
The properties “Softmax-Free”, “Exp-Free (exponential-function-free)”, and “FP-Mul-Free (floating-point multiplication-free)” are assessed at the attention-module level.
}
\centering
\scriptsize
\setlength{\tabcolsep}{5pt}
\renewcommand{\arraystretch}{1.1}
\begin{tabular}{lcccccccc}
\hline
\textbf{Model} &
\textbf{Params} &
\textbf{Spike Form} &
\textbf{Training Method} &
\textbf{Architecture} &
\textbf{Softmax-Free} &
\textbf{Exp-Free} &
\textbf{FP-Mul-Free} \\
\hline
SpikeBERT~\tiny{\citep{lv2023spikebert}}
& 109M
& Binary
& Knowledge distillation
& Encoder-only
& \checkmark
& \checkmark
& \(\times\) \\

SpikingBERT~\tiny{\citep{bal2024spikingbert}}
& 50M
& Binary
& Knowledge distillation
& Encoder-only
& \(\times\)
& \(\times\)
& \(\times\) \\

SpikeLM (BERT)~\tiny{\citep{xing2024spikelm}}
& 110M
& Ternary
& Train from scratch
& Encoder-only
& \(\times\)
& \(\times\)
& \(\times\) \\

SpikeLM (BART)~\tiny{\citep{xing2024spikelm}}
& 139M
& Ternary
& Train from scratch
& Encoder--Decoder
& \(\times\)
& \(\times\)
& \(\times\) \\

SpikeLM (mBART)~\tiny{\citep{xing2024spikelm}}
& 680M
& Ternary
& Train from scratch
& Encoder--Decoder
& \(\times\)
& \(\times\)
& \(\times\) \\

SpikeGPT~\tiny{\citep{zhu2023spikegpt}}
& 216M
& Binary
& Train from scratch
& Decoder-only
& \checkmark
& \(\times\)
& \(\times\) \\

SpikeLLM~\tiny{\citep{xing2024spikellm}}
& 7-70B
& Integer
& Post-training quantization
& Decoder-only
& \(\times\)
& \(\times\)
& \(\times\) \\

BiSpikCLM (ours)
& 1.3B
& Binary
& SpAD
& Decoder-only
& \checkmark
& \checkmark
& \checkmark \\

\hline
\end{tabular}
\label{tab:spiking_lm_attention_comparison}
\end{table*}

\section{Related Work}

\subsection{SNNs in Downstream Tasks}
Recent works show SNNs achieving competitive performance in vision tasks with lower computational consumption. In image classification, advances in surrogate gradients, attention have boosted accuracy and efficiency on CIFAR-10/100 and ImageNet~\citep{rathi2020enabling,zhou2022spikformer,zhou2023spikingformer,zhou2024qkformer,li2024spikeformer}. In object detection, models like SFOD and SpikeYOLO reduce energy cost while closing the gap with ANNs~\citep{su2023deep,bodden2024spiking,fan2024sfod,luo2024integer}. For event-based vision, architectures such as MG-SNN effectively handle gesture, motion, and optical flow tasks~\citep{lee2020spike,gehrig2021dsec,qiu2025efficient,zheng2025motion}.

In contrast, the application of SNNs in natural language processing (NLP) is still largely underexplored, with only a few attempts adapting language models to spike-based computation. For example, ~\citet{lv2023spikebert} employs a two-stage distillation pipeline to train an encoder-only spiking student from a pre-trained BERT teacher; similarly, ~\citet{bal2024spikingbert} distills a compact 4-block spiking encoder via two-stage KD. However, both still retain floating-point operations and are evaluated at modest scales. ~\citet{xing2024spikelm} proposes a spike-driven language model with bi-directional encoding, yet relies on floating-point spikes and retains softmax operations, undermining event-driven efficiency. ~\citet{zhu2023spikegpt} replaces attention with a linear-complexity Spiking RWKV module, but still depends on dense floating-point computation and remains modest in size (216M parameters). ~\citet{xing2024spikellm} scales to 7B--70B with the GIF neuron and OBSpiking framework, but substitutes binary spikes with quantized integer signals and retains softmax operation, losing event-driven benefits and fine-grained temporal dynamics; moreover, its attention does not incorporate causal masking adapted to SNN timing constraints. These prior approaches differ substantially in spike representation, training strategy, and especially attention design. A detailed comparison at the attention-module level is provided in Table~\ref{tab:spiking_lm_attention_comparison}.

\subsection{Knowledge Distillation}
Knowledge distillation is a widely adopted approach for compressing large-scale language models into smaller, more efficient ones, as demonstrated by models like DistilBERT and TinyBERT~\citep{sanh2019distilbert,jiao2019tinybert}. In the context of SNNs, early distillation efforts have primarily targeted small-scale vision tasks, using spike-based student networks guided by soft targets from ANN teachers~\citep{xu2023constructing,qiu2024self,xu2024reversing}.

In contrast, spike-based distillation for language modeling remains underexplored. Existing methods often overlook the temporal dynamics of SNNs or lack alignment in the spiking domain. For example, SpikeBERT~\citep{lv2023spikebert} maps spike activations into continuous representations via an additional MLP for teacher-student alignment. However, this introduces extra trainable parameters and computational overhead, while bypassing the native spike representation, thus limiting the preservation of spike-driven semantics. To address this, we propose the Spike-Aware Alignment Distillation framework tailored for spiking LLMs, featuring spike-attention and spike-feature alignment modules that enable hierarchical knowledge transfer while preserving the discrete and temporal nature of spiking computation.

\begin{figure*}[t]
\centering
\includegraphics[width=0.95\linewidth]{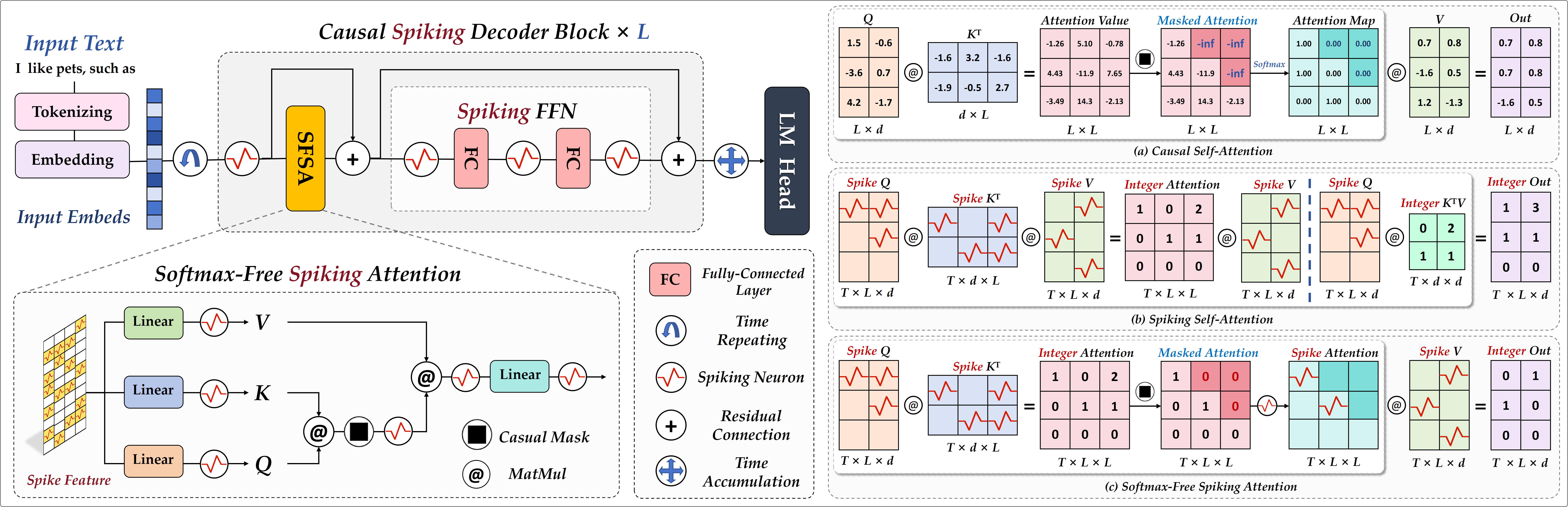}
\caption{Left depicts the BiSpikCLM framework, detailing the operations of the Softmax-Free Spiking Attention (SFSA) module and the Spiking Feed-Forward Network (SFFN) module.
Right compares the computational process of vanilla Causal Self-Attention (CSA), Spiking Self-Attention (SSA), and SFSA, where red spikes represent binary values of 1 and all other values are 0.}
\label{fig:model}
\end{figure*}

\section{Methods}
\label{method}
We propose BiSpikCLM, a binary spike-based causal language model that integrates a spike-native architectural design with an efficient training paradigm tailored for large-scale SNNs. Specifically, we design a fully spike-driven attention mechanism, Softmax-Free Spiking Attention (SFSA), which replaces conventional softmax-based attention with spike-compatible computation, supporting autoregressive language modeling using binary spikes. Building upon this architecture, we further develop Spike-Aware Alignment Distillation (SpAD), a hierarchical distillation framework that enables stable and scalable training from random initialization by transferring rich supervision signals from a frozen ANN teacher to the SNN student. The overall model architecture is shown in Figure~\ref{fig:model}, and the training strategy is illustrated in Figure~\ref{fig:distill}.

\subsection{Problem Statement}
We consider the task of autoregressive generation using a decoder-only causal language model (CLM). Formally, given a sequence of tokens \( x_1,x_2, \dots, x_n \), the model is trained to predict the next token \( x_{n+1}\) conditioned on the previous \( n \) tokens. This can be expressed as maximizing the likelihood \( P(x_{n+1} \mid x_1, x_2, \dots, x_n) \).

During the pre-training stage, the ground-truth label for each autoregressive generation step \( \tau \) is the token \( x_{\tau+1} \), and the model is optimized using the standard cross-entropy loss. The goal is to learn a function that maps token sequences to probability distributions over the vocabulary, employing causal (unidirectional) attention under temporal constraints.

\subsection{BiSpikCLM Architecture}
To enable efficient language modeling with SNNs, we propose \textbf{BiSpikCLM}, which integrates binary spiking neurons with causal attention for softmax-free, energy-efficient computation. Unlike prior works~\citep{zhu2023spikegpt,xing2024spikelm,schmidgall2024brain,xing2024spikellm}, BiSpikCLM is fully spike-driven and employs a Hadamard-masked dot product followed by spiking neuron to implement causal attention without softmax. The architecture consists of three main components: (1) Spiking Neuron Modules, (2) Softmax-Free Spiking Attention (SFSA), and (3) Spiking Feed-Forward Network (SFFN). The overall design is built upon the OPT-family architecture~\citep{zhang2022opt}, chosen for its open-source nature, simplicity, and proven effectiveness.

\subsubsection{Spiking Neuron Modules}
We primarily instantiate our spike-driven causal language model with the standard Leaky Integrate-and-Fire (LIF) neuron~\citep{wu2018spatio}, which strikes a favorable balance between computational efficiency and modeling capacity. This design results in our core model, \textbf{BiSpikCLM}, where neurons emit binary spikes in $\{0,1\}$. The LIF implementation follows SpikingJelly~\citep{fang2023spikingjelly}, a widely adopted framework for efficient spiking neural networks.

Concretely, the LIF neuron generates a spike \( S_t \in \{0,1\} \) when the membrane potential \( U_t \) exceeds a firing threshold \( U_{\mathrm{thr}} \), followed by a reset:
\begin{equation}
\begin{aligned}
S_t &=
\begin{cases}
1, & \text{if } U_t \geq U_{\mathrm{thr}}, \\
0, & \text{otherwise},
\end{cases} \\
U_t &= I_t + \beta U_{t-1} - S_{t-1} U_{\mathrm{thr}},
\end{aligned}
\label{eq:lif_spike}
\end{equation}
where \( I_t = W X_t \) denotes the input current and \( \beta \) controls the temporal decay of the membrane potential. This mechanism induces sparse, event-driven computation and avoids dense activation patterns in ANN counterparts.

Importantly, our architecture is neuron-agnostic and can be readily extended to different spiking formulations. To demonstrate this flexibility, we further introduce \textbf{TriSpikCLM}, which replaces the binary LIF neuron with a ternary spiking neuron inspired by~\citep{xing2024spikelm}. In this variant, spike outputs are extended to ternary values \( \{-\alpha, 0, +\alpha\} \), allowing neurons to encode both the sign and magnitude of membrane potential deviations.

Specifically, the ternary neuron produces:
\begin{equation}
S_t =
\begin{cases}
-\alpha, & \text{if } U_t < -\alpha, \\
0,       & \text{if } |U_t| \leq \alpha, \\
+\alpha, & \text{if } U_t > +\alpha,
\end{cases}
\label{eq:ternary_spike}
\end{equation}
with the membrane potential updated as:
\begin{equation}
U_t = U_t(\alpha - S_t) + U_{\mathrm{reset}} S_t .
\label{eq:amplitude_update}
\end{equation}

While TriSpikCLM provides richer signal representation by relaxing binary constraints, it incurs additional computational overhead and weakens the strict sparsity and event-driven efficiency characteristic of binary SNNs. We therefore include TriSpikCLM mainly as a comparative extension to assess the trade-off between representational expressiveness and efficiency.

\begin{figure*}[t]
\centering
\includegraphics[width=0.95\linewidth]{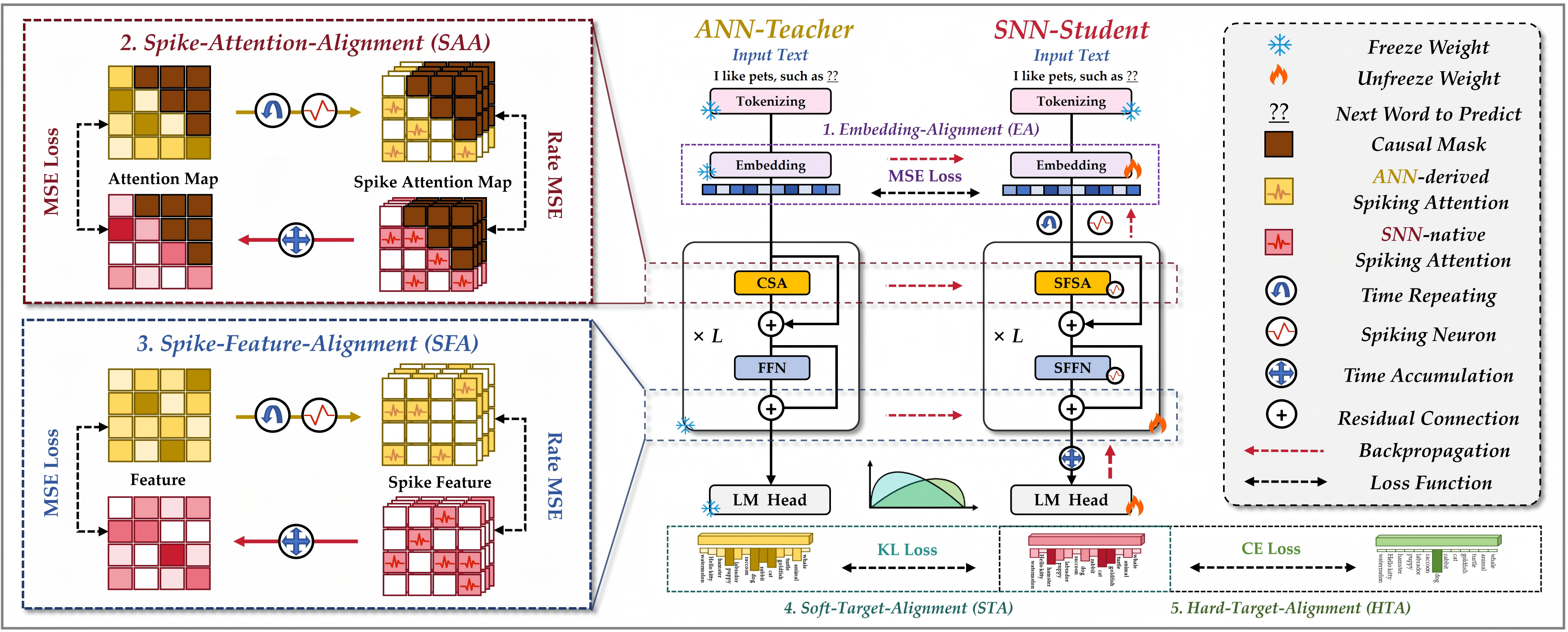}
\caption{
Overview of our \textbf{Spike-Aware Alignment Distillation (SpAD)} framework. 
Knowledge is transferred from a frozen ANN teacher to a trainable SNN student via five alignment modules: 
(1) \textbf{Embedding Alignment} (EA); 
(2) \textbf{Spike-Attention Alignment} (SAA); 
(3) \textbf{Spike-Feature Alignment} (SFA); 
(4) \textbf{Soft-Target Alignment} (STA); and 
(5) \textbf{Hard-Target Alignment} (HTA). 
Losses include MSE, CE, and spike-aware temporal strategies. 
In particular, our proposed \textit{Rate-MSE} loss in Eq.~\eqref{eq:rate_mse} aligns the attention dynamics between ANN and SNN models over time. Dashed arrows indicate loss paths; spike-related operations are denoted with icons.}
\label{fig:distill}
\end{figure*}

\subsubsection{Softmax-Free Spiking Attention (SFSA)}

To enable attention mechanisms in spike-based neural networks while preserving computational efficiency, we propose the \textbf{Softmax-Free Spiking Attention (SFSA)} module, presented in Figure~\ref{fig:sfsa}. SFSA reformulates the classical self-attention mechanism using spike-based representations, constrained by causality and spiking dynamics. 

Specifically, input spike sequences are projected to queries, keys, and values, which are then discretized via spiking neurons. Spike-based dot products between queries and keys yield integer-valued attention scores, followed by a causal mask to ensure autoregressive flow. The masked scores are passed through a spiking neuron to produce sparse attention weights, which are used to compute the weighted sum over value spikes. A final projection and spiking activation generate the output. Details are provided in Appendix~\ref{appendix:sfsa}.

\subsubsection{Spiking Feed-Forward Network (SFFN)}

The \textbf{Spiking Feed-Forward Network (SFFN)} replaces nonlinear activations in the standard Transformer FFN with temporal spiking neurons:
\begin{equation}
\mathcal{FC}(x) = \text{SN}\big(\mathcal{W} x + b \big),
\end{equation}
\begin{equation}
\mathrm{SFFN}(x) = \mathcal{FC}_2 \big( \mathcal{FC}_1(x) \big),
\end{equation}
where \text{SN} denotes a spiking neuron. While binary SFFN preserves strict event-driven sparsity, the ternary extension enables richer signed encoding via multi-level discrete states.

\subsection{Spike-Aware Alignment Distillation}
To enable effective knowledge transfer from an ANN teacher to an SNN student, we propose \textbf{Spike-Aware Alignment Distillation (SpAD)}, consisting of five key components targeting different levels (Figure~\ref{fig:distill}). To address structural mismatches between teacher and student (e.g., embedding size, depth, or number of attention heads), we employ lightweight structural alignment techniques such as linear projection, head-wise mapping, and layer skipping.

Among the proposed alignments, we focus on \textbf{Spike-Attention Alignment} and \textbf{Spike-Feature Alignment} to bridge the core ANN–SNN gap: continuous-valued representations versus discrete, temporally extended spike activities. Detailed formulations are provided in Appendix~\ref{appendix:SpAD}.

\paragraph{Spike-Attention Alignment}
Prior work~\cite{ji2021show, wang2021knowledge, gou2023hierarchical} on knowledge distillation has consistently shown that aligning intermediate representations, especially attention maps and hidden features, is critical for transferring structural and relational knowledge, beyond matching output distributions alone.

Given the difference in attention mechanisms, floating-point representations in the ANN versus spike-based representations in the SNN, as well as the additional temporal dimension in the SNN, which causes dimensional mismatch, we design two alignment strategies to enable effective cross-domain knowledge transfer:

\textbf{(a) Temporal Replication and Spike Encoding:} 
We replicate the attention map $A_{\text{ANN}}\in\mathbb{R}^{L\times L}$ across $T$ time steps:
\begin{equation}
\tilde{A}_{\text{ANN}} = \text{Repeat}(A_{\text{ANN}}, T) \in \mathbb{R}^{T \times L \times L}.
\end{equation}

We then convert $\tilde{A}_{\text{ANN}}$ into a spike-form attention using the same spiking neuron module as the SNN student, denoted by $\sigma_{\text{spike}}$ and applied entry-wise. 
Concretely, for each entry $(i,j)$ we treat $a_{ij}=A_{\text{ANN}}(i,j)$ as a constant input current and run a spiking neuron over $T$ time steps:

\begin{equation}
\begin{aligned}
U_t^{(i,j)} &= a_{ij} + \beta U_{t-1}^{(i,j)} - S_{t-1}^{(i,j)} U_{\mathrm{thr}}, \\
S_t^{(i,j)} &=
\begin{cases}
1, & \text{if } U_t^{(i,j)} \ge U_{\mathrm{thr}}, \\
0, & \text{otherwise}.
\end{cases}
\end{aligned}
\label{eq:sigma_spike_def}
\end{equation}

where $\beta\in(0,1)$ is the leak factor and $U_{\mathrm{thr}}$ is the firing threshold.

We define $\hat{A}_{\text{spike}}^{\text{ANN}}(t)_{ij} = S_t^{(i,j)}$, i.e.,

\begin{equation}
\hat{A}_{\text{spike}}^{\text{ANN}}(t)
= \sigma_{\text{spike}}\!\left(\tilde{A}_{\text{ANN}}(t)\right),
\quad t=1,\ldots,T.
\end{equation}

We compute Rate-MSE loss:
\begin{equation}
\mathcal{L}_{\text{attn}}^{\text{RateMSE}} = \text{MSE}\left(\frac{1}{T} \sum_{t=1}^{T}
 \hat{A}_{\text{spike}}^{\text{ANN}}(t), \frac{1}{T} \sum_{t=1}^{T}
 A_{\text{SNN}}(t)\right).
\label{eq:rate_mse}
\end{equation}

\textbf{(b) Temporal Fusion and Distribution Matching:} Alternatively, we temporally average the SNN attention and match it to the ANN attention using Mean Squared Error (MSE):
\begin{equation}
\bar{A}_{\text{SNN}} = \frac{1}{T} \sum_{t=1}^{T} A_{\text{SNN}}(t)
\label{eq:snn_avg}
\end{equation}

\begin{equation}
\mathcal{L}_{\text{attn}}^{\text{MSE}} = \text{MSE}(A_{\text{ANN}}, \bar{A}_{\text{SNN}})
\label{eq:attn_mse}
\end{equation}

The overall spike-attention alignment loss is:
\begin{equation}
\mathcal{L}_{\text{attn}} = \gamma_{\text{attn}} \mathcal{L}^{\text{RateMSE}}_{\text{attn}} + (1-\gamma_{\text{attn}})\, \mathcal{L}^{\text{MSE}}_{\text{attn}},
\end{equation}
where $\gamma_{\text{attn}} \in (0,1)$ balances the contribution of rate-based and temporally averaged alignment about attention scores.

\paragraph{Rationale of method (a)}
Temporal Replication and Spike Encoding are motivated by the rate-coding principle in biological and artificial spiking neurons. When a spiking neuron is driven by a constant input $a$, its long-term firing rate converges to a stable value that is a bounded and non-decreasing function of $a$. 

As a result, replicating $A_{\mathrm{ANN}}$ over time and passing it through the spiking neuron module $\sigma_{\text{spike}}$ yields a spiking domain proxy, i.e., a spike-based representation whose time-averaged activity approximately preserves the relative ordering and magnitude of the original attention values.

Formally, let
\begin{equation}
\hat r_T(a) = \frac{1}{T}\sum_{t=1}^{T} S_t
\end{equation}
denote the empirical firing rate induced by a constant input current $a$.
Under the discrete-time LIF dynamics defined in Eq.~\eqref{eq:lif_spike}, when driven by a constant input $I_t = a$ for all $t$, 
the membrane potential evolves periodically once spiking occurs.
As $T \to \infty$, the empirical rate $\hat r_T(a)$ converges to a stationary firing rate $r(a)$.

Moreover, there exist constants $0 \le r_{\min} \le r_{\max} \le 1$ such that
\begin{equation}
r_{\min} \le r(a) \le r_{\max},
\end{equation}
and for any $a_1 \le a_2$,
\begin{equation}
r(a_1) \le r(a_2),
\end{equation}
i.e., $r(a)$ is a bounded and monotone non-decreasing function of the input drive $a$.

While the above analysis considers the asymptotic regime $T \to \infty$, in practice we use a small number of time steps (e.g., $T=2$ or $4$). Importantly, Rate-MSE only requires an order-preserving correspondence between input drives and spike activities, where larger inputs yield higher spike probabilities even with few time steps. And stochastic averaging across tokens, heads, and iterations further mitigates variance.

Therefore, the mapping $a \mapsto \hat r_T(a)$ provides an approximately order-preserving spiking domain proxy of ANN attention signals, and minimizing Rate-MSE in Eq.~\eqref{eq:rate_mse} encourages the student SNN to match the teacher’s attention structure. A detailed dynamical analysis is provided in Appendix~\ref{appendix:SpAD}.

\paragraph{Rationale of method (b)}
Method (b) complements method (a) by aligning the \emph{overall} attention structure. While method (a) enforces temporal compatibility, method (b) matches the student’s cumulative attention to the teacher’s static map. We compute an effective SNN attention by averaging over time (Eq.~\eqref{eq:snn_avg}), turning binary spikes into a rate-coded representation, where each entry reflects its relative importance over the simulation window. This form preserves fine-grained token-to-token differences while smoothing spike noise. Minimizing the MSE between the rate-coded attention and \(A_{\text{ANN}}\) encourages the student to attend to similar semantic cues as the teacher.

\paragraph{Spike-Feature Alignment}

Similarly, intermediate hidden states are aligned using a combination of rate-based and temporally averaged MSE:
\begin{equation}
\mathcal{L}_{\text{feat}} = \gamma_{\text{feat}} \mathcal{L}^{\text{RateMSE}}_{\text{feat}} + (1-\gamma_{\text{feat}})\, \mathcal{L}^{\text{MSE}}_{\text{feat}}, 
\end{equation}
where $\gamma_{\text{feat}} \in (0,1)$ balances the contribution of rate-based and temporally averaged alignment about intermediate features.

\paragraph{Total Loss}

Combining embedding alignment, spike-based alignments, and traditional distillation, the student SNN is supervised with the overall training objective:
\begin{equation}
\label{Loss}
\mathcal{L}_{\text{total}} = \lambda_1 \mathcal{L}_{\text{emb}} + \lambda_2 \mathcal{L}_{\text{attn}} + \lambda_3 \mathcal{L}_{\text{feat}} + \lambda_4 \mathcal{L}_{\text{soft}} + \lambda_5 \mathcal{L}_{\text{hard}}.
\end{equation}
The coefficients $\{\lambda_i\}_{i=1}^{5}$ are listed in Table~\ref{tab:training_hyperparams} and satisfy $\sum_{i=1}^{5}\lambda_i=1$ for a stable loss scale; we set them based on preliminary sweeps, with larger weights on distillation terms improving convergence and accuracy.

\begin{table*}[h]
\caption{Comparison of performance and estimated energy efficiency between BiSpikCLM, TriSpikCLM and conventional ANN baselines on the ACC benchmark. 
BiSpikCLM adopts classic LIF neurons (see Eq.~\eqref{eq:lif_spike}) implemented via SpikingJelly~\citep{fang2023spikingjelly}, while TriSpikCLM employs ternary-valued spiking neurons with amplitude encoding (see Eq.~\eqref{eq:ternary_spike}). 
All energy estimates are calculated under a uniform FP32-based energy model for fair comparison. Time Step indicates the number of discrete simulation steps during SNN inference.}
\label{tab:results}
\centering
\footnotesize
\setlength{\tabcolsep}{2.5pt} 
\renewcommand{\arraystretch}{1.05} 
\begin{tabular}{l | c c c c c c c | *{9}{c}}
\toprule
{\multirow{2}{*}{\textbf{Model}}} & 
{\multirow{2}{*}{\textbf{\makecell{Params \\ (B)}}}} & 
{\multirow{2}{*}{\textbf{\makecell{Tokens \\ (B)}}}} & 
{\multirow{2}{*}{\textbf{\makecell{Spike \\ Form}}}} & 
{\multirow{2}{*}{\textbf{\makecell{Time \\ Step}}}} & 
{\multirow{2}{*}{\textbf{\makecell{OPs \\ (G)}}}} & 
{\multirow{2}{*}{\textbf{\makecell{Firing \\ Rate}}}} & 
{\multirow{2}{*}{\textbf{\makecell{Energy \\ (mJ)}}}} & 
\multicolumn{9}{c}{\textbf{Zero - shot Accuracy (\%) $\uparrow$}} \\
& & & & & & & & {\scriptsize \textbf{ARC-e}} & {\scriptsize \textbf{ARC-c}} & {\scriptsize \textbf{WG}} & {\scriptsize \textbf{BQ}} & {\scriptsize \textbf{PIQA}} & {\scriptsize \textbf{HS}} & {\scriptsize \textbf{OBQA}} & {\scriptsize \textbf{HQA}} & {\scriptsize \textbf{Avg.}} \\
\midrule
OPT                  & 0.125 & 180 & $\times$ & — & 125.6  & —     & 125.95 & 43.6 & 19.3 & 52.3 & 54.6 & 62.4 & 32.1 & 20.2 & 23.7 & 38.60 \\
Pythia               & 0.160 & 300 & $\times$ & — & 125.7  & —     & 126.01 & 43.7 & 19.8 & 52.8 & 55.1 & 62.7 & 33.6 & 20.1 & 24.2 & 39.00 \\
SpikeGPT        & 0.046   & 16.5 & Binary & — & 3.66   & 0.174 & 3.29 & 32.3 & 16.2 & 50.2 & 45.7 & 54.6 & 25.3 & 15.7 & 20.6 & 32.58 \\
SpikeGPT        & 0.216   & 16.5 & Binary & — & 18.3   & 0.168 & 16.53 & 35.2 & 17.7 & 50.7 & 47.3 & 55.1 & 27.6 & 17.3 & 23.1 & 34.25 \\

BiSpikCLM  & 0.125 & \textbf{1.0} & Binary & 2 & 12.1   & 0.196 & \textbf{5.24}   & 39.1 & 18.9 & 50.3 & 52.7 & 56.7 & 28.1 & 19.8 & 22.9 & 36.05 \\
TriSpikCLM  & 0.125 & \textbf{1.0} & Ternary & 2 & 13.7   & 0.412 & \textbf{10.74}  & 38.5 & 18.3 & 51.3 & 52.3 & 57.7 & 29.1 & 19.2 & 22.5 & 36.11 \\
BiSpikCLM  & 0.125 & \textbf{1.0} & Binary & 4 & 23.1   & 0.173 & \textbf{9.43}   & 39.4 & 19.0 & 51.2 & 53.0 & 57.5 & 29.2 & 19.7 & 23.1 & 36.50 \\
TriSpikCLM  & 0.125 & \textbf{1.0} & Ternary & 4 & 25.8   & 0.386 & \textbf{19.92}  & 38.9 & 18.5 & 51.5 & 52.9 & 58.0 & 28.3 & 19.2 & 22.9 & 36.27 \\
\midrule
OPT                  & 0.350 & 180 & $\times$ & — & 360.8  & —     & 197.57 & 47.5 & 22.2 & 55.3 & 57.2 & 66.1 & 40.7 & 25.7 & 26.6 & 42.68 \\
Pythia               & 0.410 & 300 & $\times$ & — & 360.9  & —     & 197.71 & 48.7 & 24.8 & 56.8 & 58.1 & 66.7 & 41.6 & 26.1 & 26.2 & 43.63 \\
BiSpikCLM  & 0.350 & \textbf{2.0} & Binary& 2 & 43.4   & 0.182 & \textbf{9.31}   & 41.5 & 21.7 & 52.3 & 55.1 & 59.7 & 32.7 & 21.2 & 23.8 & 38.48 \\
TriSpikCLM  & 0.350 & \textbf{2.0} & Ternary & 2 & 47.7   & 0.404 & \textbf{18.61}  & 41.5 & 21.7 & 52.4 & 54.9 & 59.1 & 31.6 & 20.8 & 23.1 & 38.14 \\
BiSpikCLM  & 0.350 & \textbf{2.0} & Binary & 4 & 84.2   & 0.178 & \textbf{16.75}  & 42.1 & 21.4 & 52.1 & 56.1 & 60.5 & 33.1 & 21.9 & 23.5 & 38.84 \\
TriSpikCLM  & 0.350 & \textbf{2.0} & Ternary & 4 & 88.3   & 0.377 & \textbf{35.35}  & 41.8 & 21.9 & 52.8 & 55.7 & 60.4 & 34.0 & 21.3 & 23.1 & 38.87 \\
\midrule
OPT                  & 1.300 & 180 & $\times$ & — & 1237.1 & —     & 632.22 & 57.8 & 30.4 & 60.4 & 60.8 & 71.7 & 52.6 & 33.4 & 30.7 & 49.73 \\
Pythia               & 1.400 & 300 & $\times$ & — & 1237.4 & —     & 632.48 & 60.5 & 31.2 & 61.3 & 61.1 & 71.1 & 53.6 & 33.2 & 31.9 & 50.49 \\
BiSpikCLM  & 1.300 & \textbf{10.0} & Binary & 2 & 66.7   & 0.192 & \textbf{37.16}  & 45.7 & 23.5 & 54.2 & 56.3 & 62.3 & 40.2 & 24.5 & 24.0 & 41.33 \\
TriSpikCLM  & 1.300 & \textbf{10.0} & Ternary & 2 & 74.2   & 0.426 & \textbf{74.32}  & 44.5 & 23.7 & 54.2 & 55.3 & 62.3 & 40.4 & 24.6 & 23.6 & 41.08 \\
BiSpikCLM  & 1.300 & \textbf{10.0} & Binary & 4 & 131.9  & 0.184 & \textbf{66.89}  & 46.3 & 24.3 & 55.6 & 56.8 & 63.4 & 41.7 & 25.2 & 24.3 & 42.19 \\
TriSpikCLM  & 1.300 & \textbf{10.0} & Ternary & 4 & 141.6  & 0.411 & \textbf{141.21} & 45.8 & 24.5 & 55.6 & 56.1 & 63.4 & 41.3 & 25.5 & 24.8 & 42.12 \\
\bottomrule
\end{tabular}
\end{table*}

\section{Experiments}

\subsection{Training Details}
We use \textbf{FineWeb-Edu}~\citep{penedo2024fineweb}, a high-quality subset of the FineWeb corpus curated for factual and educational content. A 10B-token portion of the dataset is selected for pretraining. Notably, our BiSpikCLM models achieve competitive performance under strict energy constraints, despite being trained on orders-of-magnitude fewer tokens (1–10 billion) compared to conventional ANN counterparts (typically requiring more than 100 Billion tokens), even at reduced parameter scales (0.125B-1.3B). The detailed training setup are provided in the Appendix~\ref{appendix:training}.

\subsection{Model Evaluation}
We evaluate models using zero-shot accuracy on diverse commonsense reasoning and QA benchmarks, including ARC-Easy (ARC-e), ARC-Challenge (ARC-c)~\citep{clark2018think}, Winogrande (WG)~\citep{sakaguchi2021winogrande}, BoolQ (BQ)~\citep{clark2019boolq}, PIQA~\citep{bisk2020piqa}, HellaSwag (HS)~\citep{zellers2019hellaswag}, OpenBookQA (OBQA)~\citep{mihaylov2018can}, and HeadQA (HQA)~\citep{vilares2019head}, measuring the generalization and reasoning abilities without task-specific finetuning. As shown in Table~\ref{tab:results}, BiSpikCLM achieves 82.60–94.56\% of the zero-shot accuracy of counterpart ANN models at the same scale, despite using significantly fewer operations and training tokens. For example, BiSpikCLM (1.3B, 10B tokens, 4 steps) reaches 42.19\% accuracy versus 49.73\% for OPT-1.3B, consuming only ~10.6\% of the energy per inference. For fair comparison with existing spiking LLMs, we focus on~\citet{zhu2023spikegpt}, a decoder-only SNN trained from scratch and architecturally comparable. We don't directly compare with ~\citet{lv2023spikebert,bal2024spikingbert,xing2024spikelm}, which are not decoder-only and target different downstream tasks. And since ~\citet{xing2024spikellm} is derived via quantization and spiking conversion from pretrained ANN LLMs, we defer detailed comparisons with such quantization-based approaches to Appendix~\ref{appendix:spikellm}.

\subsection{Energy Consumption}
To assess the efficiency of SNNs, we first measure the firing rate, defined as the average proportion of active spikes, where lower rates indicate higher sparsity and greater energy efficiency. Based on the firing rate, we then estimate the theoretical energy consumption during inference by simulating a 45nm neuromorphic chip, following ~\citet{horowitz2014energy,kundu2021spike,yin2021accurate,kim2021optimizing}. Energy estimates are based on the total number of spike operations (SOPs), compared against floating-point operations (FLOPs) in baseline ANN models. Detailed computation steps are provided in Appendix~\ref{appendix:energy}. Table~\ref{tab:results} reports per-sample energy, firing rates, and zero-shot accuracy. BiSpikCLM consumes an order of magnitude less energy than ANN baselines (e.g., 9.43 mJ vs. 126.01 mJ at 125M) while achieving over 93\% of the accuracy. Across 0.125B–1.3B parameters, it maintains competitive performance at only \textbf{4.16\%–5.87\%} of the computational cost. Moreover, increasing time steps slightly improves performance (36.05\% → 36.50\% at 125M) with moderate energy overhead. TriSpikCLM achieves slightly higher accuracy at higher energy, providing a flexible trade-off. These findings validate the viability of SNN-based LLMs for energy-constrained environments, such as edge devices and neuromorphic accelerators.

\begin{figure}[t]
    \centering

    \begin{subfigure}{0.48\columnwidth}
        \centering
        \includegraphics[width=\linewidth]{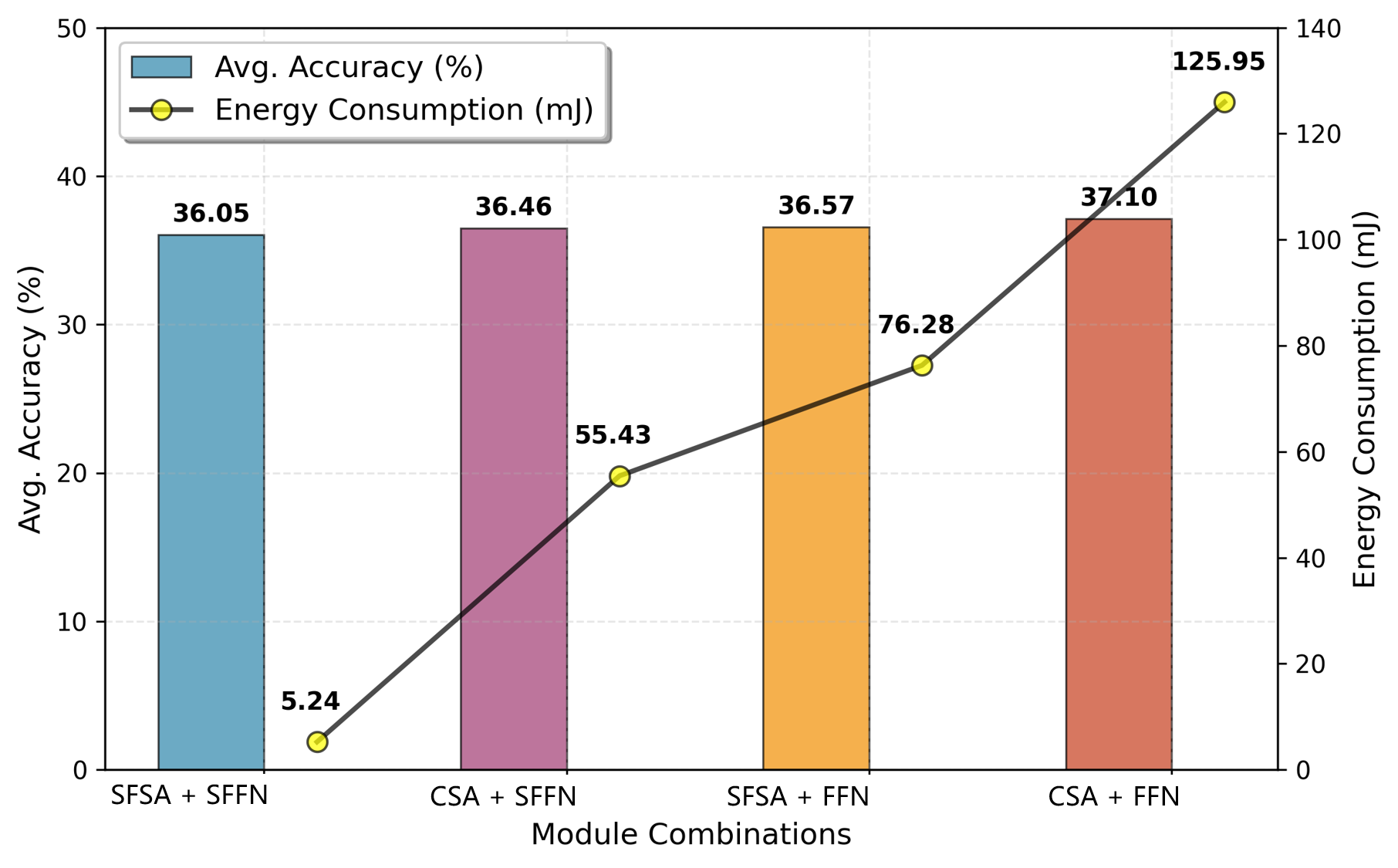}
        \caption{Spike-driven modules.}
        \label{fig:ablation1}
    \end{subfigure}
    \hfill
    \begin{subfigure}{0.48\columnwidth}
        \centering
        \includegraphics[width=\linewidth]{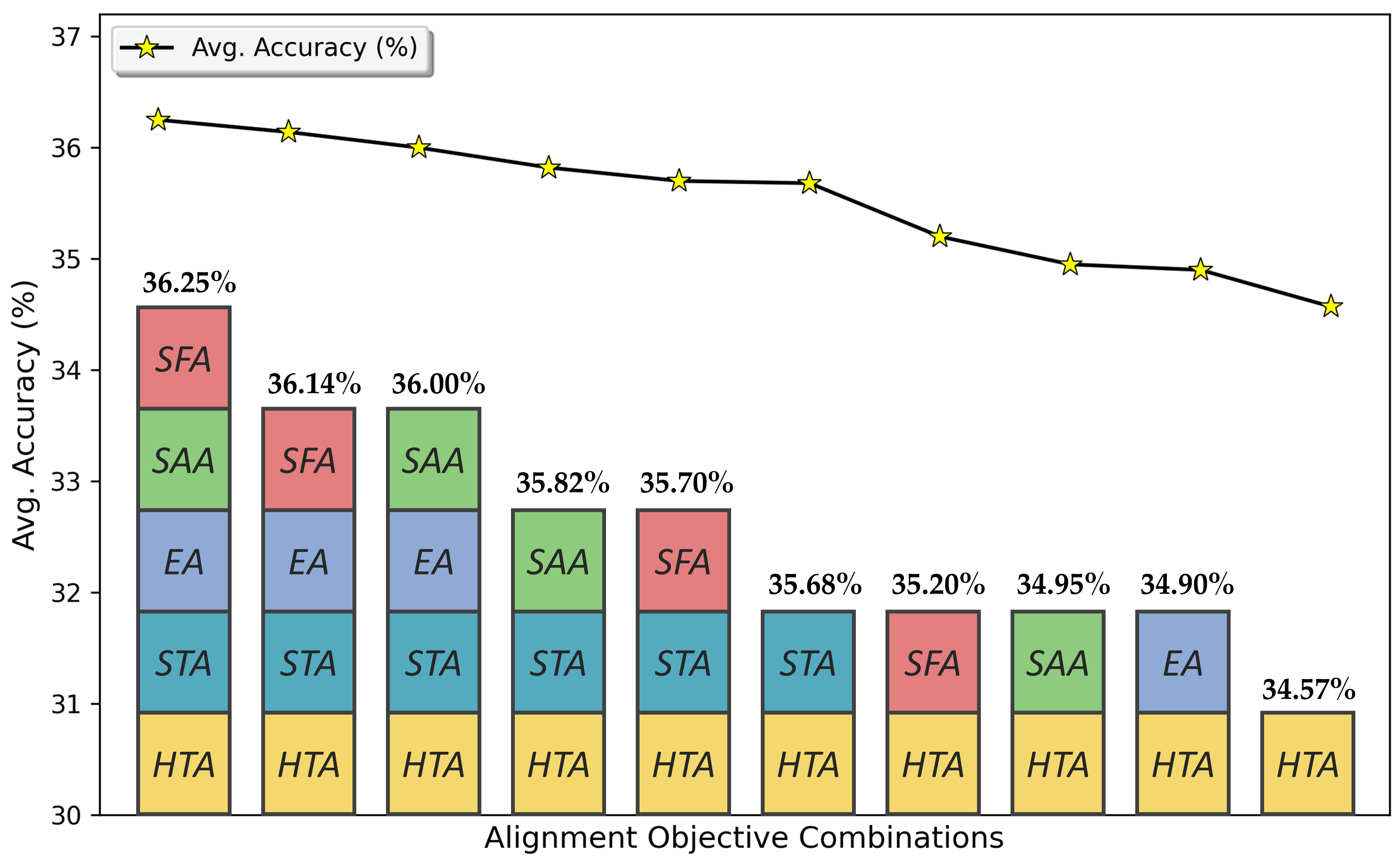}
        \caption{Alignment objectives.}
        \label{fig:ablation2}
    \end{subfigure}

    \vspace{2mm}

    \begin{subfigure}{0.48\columnwidth}
        \centering
        \includegraphics[width=\linewidth]{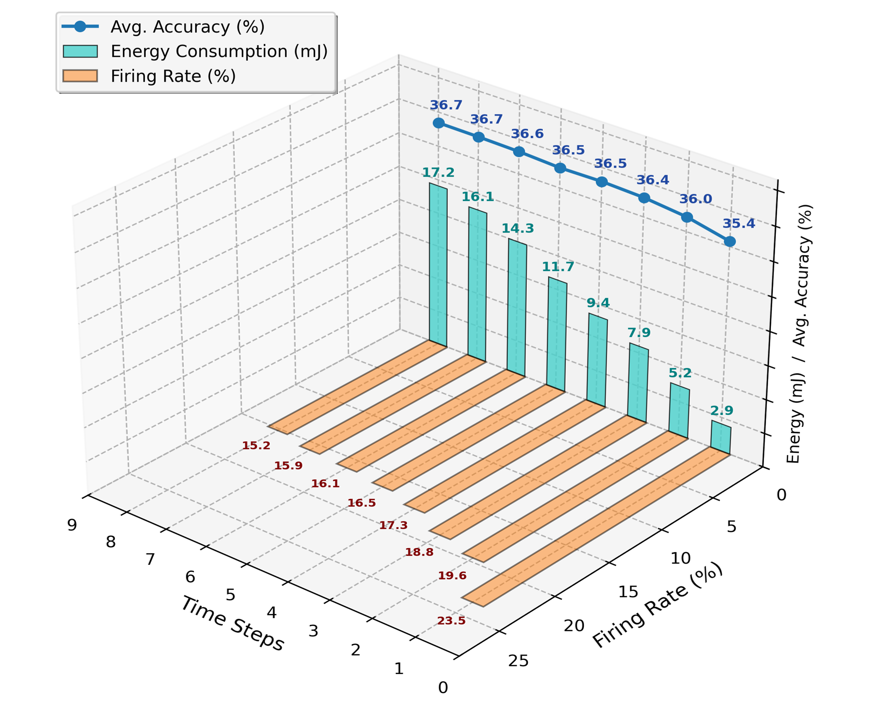}
        \caption{Varying time steps.}
        \label{fig:ablation3}
    \end{subfigure}
    \hfill
    \begin{subfigure}{0.48\columnwidth}
        \centering
        \includegraphics[width=\linewidth]{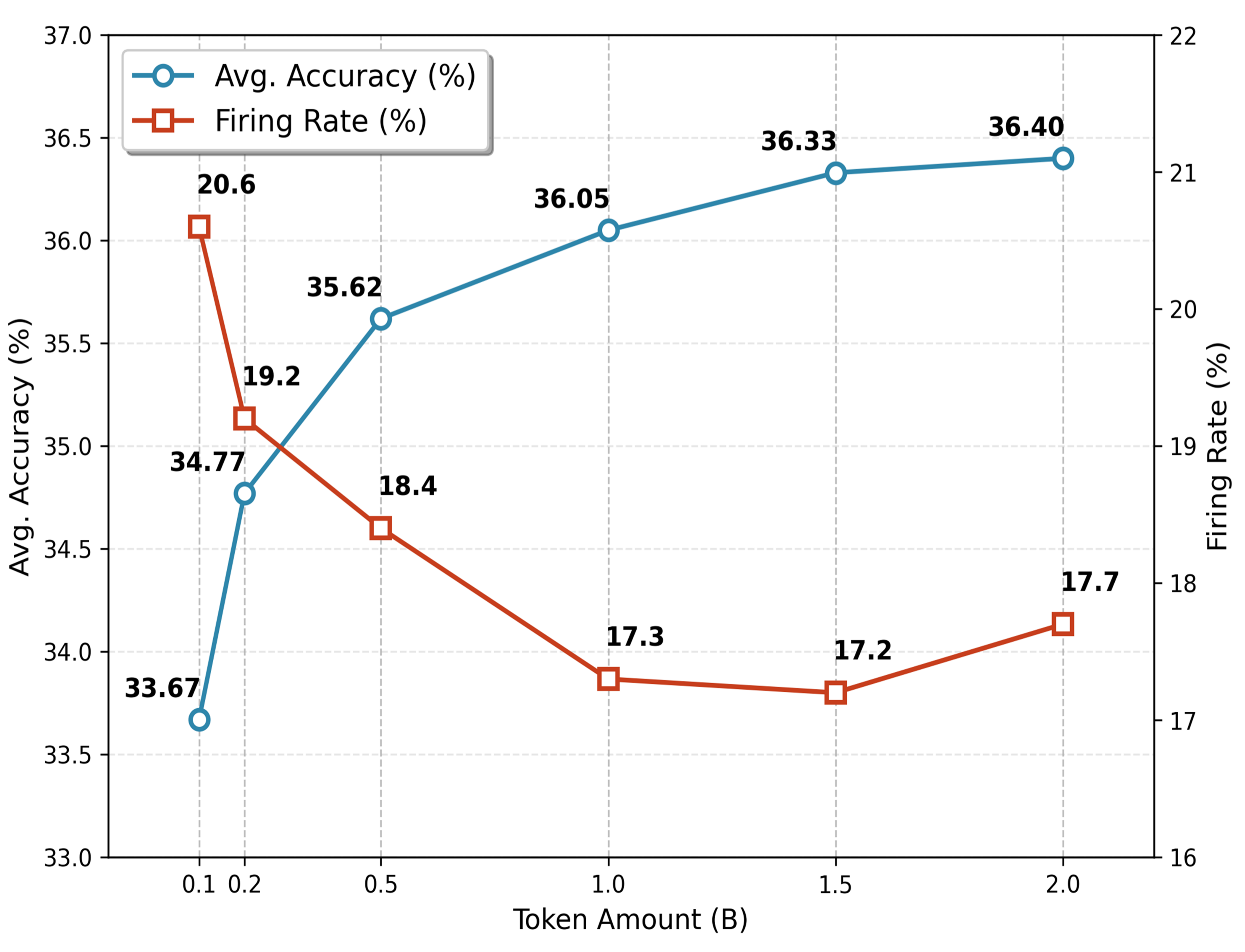}
        \caption{Training data volume.}
        \label{fig:ablation4}
    \end{subfigure}

    \caption{Visualization of ablation experiments.}
    \label{fig:ablation_overall}
\end{figure}

\subsection{Ablation Studies}
We conduct ablation experiments on the \textbf{BiSpikCLM} model (125M parameters) to evaluate the contributions of key components and training factors, specifically: (1) spike-driven modules, (2) simulation time steps, (3) training token volume, and (4) spike-aware distillation alignment.
\vspace{-0.1cm} 
\paragraph{Spike-driven Modules}
We first examine spike-driven modules by replacing SFSA and SFFN with their ANN counterparts. As shown in Figure~\ref{fig:ablation1}, the fully spike-driven design (SFSA+SFFN) achieves 36.05\% accuracy with only 5.24 mJ energy. Replacing either module slightly improves accuracy (up to 36.57\%) but increases energy consumption by more than $10\times$. Using ANN attention and FFN together yields 37.10\% accuracy at the cost of $24\times$ higher energy. These results highlight that SFSA and SFFN are crucial for preserving BiSpikCLM’s energy efficiency.
\vspace{-0.1cm} 
\paragraph{Distillation Alignment Strategy}
We assess each alignment component by progressively adding it to the base HTA model (Figure~\ref{fig:ablation2}). STA yields the largest individual gain (+1.11\%), while EA contributes less (+0.33\%), suggesting limited standalone benefit. Higher-level constraints SFA and SAA further improve performance (+0.63\% and +0.38\%). When combining STA with EA or SFA, the improvements increase more significantly, indicating complementary effects. Combining multiple objectives produces stronger gains, and the full set (STA, EA, SFA, SAA) achieves the best accuracy (36.25\%, +1.68\%), demonstrating the complementary benefits of hierarchical alignment for effective knowledge transfer from ANN teachers to SNN students.
\vspace{-0.1cm} 
\paragraph{Time Steps}
We investigate simulation steps from 1 to 8 (Figure~\ref{fig:ablation3}). Increasing time steps improves accuracy by refining temporal resolution, but gains saturate beyond 4 steps while energy cost rises sharply. Firing rates gradually decline with longer steps, indicating increased sparsity. Overall, 2–4 steps provide a good trade-off between efficiency and performance, while more steps yield marginal accuracy gains at higher energy cost.
\vspace{-0.1cm} 
\paragraph{Training Token Volume}
We evaluate the effect of training data by varying tokens from 0.1B to 2.0B (Figure~\ref{fig:ablation4}). Accuracy improves consistently, with the largest gains in the low-data regime (0.1B → 0.5B) and saturation beyond 1.0B tokens, rising from 33.67\% to 36.40\% (94.3\% of the teacher’s 38.60\%). Interestingly, the firing rate gradually decreases as token volume grows, suggesting that larger datasets not only improve performance but also enhance temporal sparsity, likely due to more structured representations. Additional firing patterns are visualized in Appendix~\ref{appendix:visual_firing}.

\section{Conclusion}
In this work, we introduce BiSpikCLM, the first fully binary spiking MatMul-free causal language model. Its Softmax-Free Spiking Attention (SFSA) enables spike-based autoregressive modeling, reducing computational cost by over 10$\times$ compared to ANNs. And Spike-Aware Alignment Distillation (SpAD) framework further improves performance by aligning representations across multiple levels. 
We further introduce TriSpikCLM as a ternary extension and comparison baseline. Results show that the fully binary design already achieves competitive accuracy with clear and consistent energy advantages, demonstrating a promising pathway for energy-efficient, brain-inspired NLP. 

\vspace{0.2cm} 

\textbf{Limitations:} While BiSpikCLM significantly reduces computational cost, its accuracy still lags behind large-scale ANN LLMs on some benchmarks, and training larger models requires careful tuning of time steps and distillation schedules. Future work could explore improved spike-based architectures and more effective distillation strategies to further close the gap with ANN performance.

\section*{Impact Statement}
This paper presents work whose goal is to advance the field of Machine
Learning. There are many potential societal consequences of our work, none 
which we feel must be specifically highlighted here.

\bibliography{example_paper}
\bibliographystyle{icml2026}

\newpage
\appendix
\onecolumn

\section{Use of LLMs}
In this work, we used Large Language Models (LLMs) in a limited and auxiliary capacity. Specifically, LLMs were employed for retrieval and discovery of related literature on Spiking Neural Networks (SNNs), neuromorphic computing, and energy-efficient large language models. This assisted us in identifying relevant prior work and ensuring broader coverage of existing approaches. Importantly, LLMs were not involved in designing algorithms, implementing models, or analyzing experimental results. All methodological innovations, including the Softmax-Free Spiking Attention (SFSA) and Spike-Aware Alignment Distillation (SpAD), were independently conceived, implemented, and validated by the authors. Thus, the role of LLMs was restricted to accelerating literature exploration, without influencing the scientific contributions of this paper.

\section{Algorithm Procedure of SFSA}
\label{appendix:sfsa}
\begingroup
\renewcommand{\baselinestretch}{1.05}
\begin{algorithm}[H]
\caption{Softmax-Free Spiking Attention (SFSA)}
\label{alg:sfsa}
\textbf{Input}: Spike-based input $X$ \\
\textbf{Output}: Spiking attention output

\begin{algorithmic}[1]
\STATE {\footnotesize \textbf{// Step 1: Input Projection ($FP \gets Spike\,@\,FP$)}}
\STATE $q, k, v \gets \text{Linear}_{Q,K,V}(X)$

\STATE {\footnotesize \textbf{// Step 2: Spiking Neuron Encoding ($Spike \gets FP$)}}
\STATE $\text{spike}_q \gets SN_{Q}(q)$ 
\STATE $\text{spike}_k \gets SN_{K}(k)$ 
\STATE $\text{spike}_v \gets SN_{V}(v)$

\STATE {\footnotesize \textbf{// Step 3: Attention ($Integer \gets Spike\,@\,Spike$)}}
\STATE $\text{attn\_int} \gets \text{spike}_q @ \text{spike}_k^{T}$

\STATE {\footnotesize \textbf{// Step 4: Causal Masking and Spiking}}
\STATE $ \text{causal\_mask} \gets \text{causal\_mask} \odot \text{attn\_mask} $
\STATE $\text{attn\_causal} \gets \text{causal\_mask} \odot \text{attn\_int}$
\STATE $spike_{attn} \gets SN_{Attn}(\text{attn\_causal})$

\STATE {\footnotesize \textbf{// Step 5: Summation ($Integer \gets Spike\,@\,Spike$)}}
\STATE $\text{attn\_out} \gets spike_{attn} @ \text{spike}_v^{T}$
\STATE $spike_{attn\_out} \gets SN_{AttnOut}(\text{attn\_out})$

\STATE {\footnotesize \textbf{// Step 6: Output Projection ($FP \gets Spike\,@\,FP$)}}
\STATE $\text{fp\_out} \gets \text{Linear}_{\text{out}}(spike_{attn\_out})$

\STATE {\footnotesize \textbf{// Step 7: Spiking ($Spike \gets FP$)}}
\STATE $\text{spike\_out} \gets SN_{Out}(\text{fp\_out})$

\STATE \textbf{return} $\text{spike\_out}$
\end{algorithmic}
\end{algorithm}
\endgroup

\section{Detailed design of SpAD}
\label{appendix:SpAD}
We present \textbf{Spike-Aware Alignment Distillation (SpAD)}, a framework that distills a frozen ANN teacher into a trainable SNN student. SpAD consists of five losses: Embedding Alignment (EA), Spike-Attention Alignment (SAA), Spike-Feature Alignment (SFA), Soft-Target Alignment (STA), and Hard-Target Alignment (HTA). EA/STA/HTA follow standard distillation practice, while SAA and SFA explicitly address the two principal mismatches between ANN and SNN:
(i) \emph{continuous vs.\ spiking} representations and (ii) \emph{static vs.\ temporal} representations.

\paragraph{Notation.}
Let $T$ be the number of spiking time steps, $L$ the sequence length, and $d$ the hidden dimension.
Teacher (ANN) representations are time-independent, while student (SNN) representations are time-indexed.
We denote attention maps by $A$ (shape $L\times L$ for teacher; $T\times L\times L$ for student) and hidden features by $H$ (shape $L\times d$ for teacher; $T\times L\times d$ for student).
If teacher and student differ in dimension (or heads/layers), we use lightweight projections and layer/head mapping to align shapes.

\subsection{Embedding Alignment (EA)}
\label{appendix:SpAD_emb}

We align the embedding outputs to ease early optimization.
Let $E_{\text{ANN}}\in\mathbb{R}^{L\times d_T}$ and $E_{\text{SNN}}(t)\in\mathbb{R}^{L\times d_S}$ be teacher and student embeddings.
We first apply a linear projection $\Pi_{\text{emb}}:\mathbb{R}^{d_S}\to\mathbb{R}^{d_T}$ if $d_S\neq d_T$ and then use temporal fusion:
\begin{equation}
\bar E_{\text{SNN}} = \frac{1}{T}\sum_{t=1}^{T}\Pi_{\text{emb}}\!\left(E_{\text{SNN}}(t)\right),\qquad
\mathcal{L}_{\text{emb}} = \left\|E_{\text{ANN}}-\bar E_{\text{SNN}}\right\|_F^2.
\end{equation}

\subsection{Spike-Attention Alignment (SAA)}
\label{appendix:SpAD_attn}
We align self-attention since it encodes relational structure (token-to-token interactions), which is difficult to transfer via logits alone.
Let $A_{\text{ANN}}\in\mathbb{R}^{L\times L}$ be the teacher attention map and $A_{\text{SNN}}(t)\in\{0,1\}^{L\times L}$ be the student spike-attention at time step $t$.

\paragraph{(a) Temporal Replication and Spike Encoding (Teacher $\rightarrow$ Student).}
To align representations in the spiking domain, we construct a spike proxy of the teacher attention.
First, we replicate the static teacher attention map across time:
\begin{equation}
\tilde{A}_{\mathrm{ANN}}(t) = A_{\mathrm{ANN}},\qquad t=1,\ldots,T.
\end{equation}
We then apply the spiking neuron $\sigma_{\text{spike}}$ entry-wise to obtain the teacher spiking attention:
\begin{equation}
\hat A_{\mathrm{spike}}^{\mathrm{ANN}}(t) = \sigma_{\text{spike}}\!\left(\tilde A_{\mathrm{ANN}}(t)\right),\qquad t=1,\ldots,T,
\end{equation}

and match time-averaged firing rates:
\begin{equation}
\bar A_{\text{spike}}^{\text{ANN}} = \frac{1}{T}\sum_{t=1}^{T}\hat A_{\text{spike}}^{\text{ANN}}(t),\qquad
\mathcal{L}^{\text{RateMSE}}_{\text{attn}} = \left\|\bar A_{\text{spike}}^{\text{ANN}}-\bar A_{\text{SNN}}\right\|_F^2.
\label{eq:attn_ratemse}
\end{equation}

\paragraph{(b) Temporal Fusion and Distribution Matching (student $\rightarrow$ teacher).}
We temporally average the student attention to remove the extra temporal dimension:
\begin{equation}
\bar A_{\text{SNN}} = \frac{1}{T}\sum_{t=1}^{T} A_{\text{SNN}}(t),\qquad
\mathcal{L}^{\text{MSE}}_{\text{attn}} = \left\|A_{\text{ANN}}-\bar A_{\text{SNN}}\right\|_F^2.
\label{eq:attn_fusion}
\end{equation}

\paragraph{Total attention alignment.}
\begin{equation}
\mathcal{L}_{\text{attn}} = \gamma_{\text{attn}} \mathcal{L}^{\text{RateMSE}}_{\text{attn}} + (1-\gamma_{\text{attn}})\, \mathcal{L}^{\text{MSE}}_{\text{attn}}.
\end{equation}

\subsubsection{Theoretical justification}
\label{appendix:SpAD_theory}
The two strategies can be justified by interpreting spiking attention as a rate-coded random variable, and temporal averaging as an estimator of its expectation.

\begin{lemma}[Entry-wise concentration under geometric mixing]
\label{lem:temporal_fusion_conc}
Fix $x$ and an entry $(i,j)$. Let $X_t = A_{\mathrm{SNN}}(t)[i,j]\in\{0,1\}$ and $\mu = \mathbb{E}[X_t\mid x]$.
Assume $\{X_t\}_{t\ge 1}$ is (conditionally on $x$) stationary and geometrically $\alpha_{\mathrm{mix}}$ with
\begin{equation}
\alpha_{\mathrm{mix}}(k) \le c_0 \exp(-k/\tau_{\mathrm{mix}})\qquad (k\ge 1),
\end{equation}
for some constants $c_0>0$ and mixing time $\tau_{\mathrm{mix}}\ge 1$
(see, e.g., \citet{doukhan1995mixing,bradley2005basic}).
Then for any $\epsilon>0$,
\begin{equation}
\Pr(|\bar X-\mu|\ge \epsilon\mid x)\;\le\; 2\exp\!\Big(-c\,\frac{T\epsilon^2}{\tau_{\mathrm{mix}}}\Big),
\end{equation}
for a universal constant $c>0$ (see, e.g., \citet{merlevede2009bernstein,rio2000tumor}). Moreover,
\begin{equation}
\mathbb{E}\big[(\bar X-\mu)^2\mid x\big] \;=\; \mathcal{O}\Big(\frac{\tau_{\mathrm{mix}}}{T}\Big),
\qquad
\mathbb{E}\big[|\bar X-\mu|\mid x\big] \;=\; \mathcal{O}\Big(\sqrt{\frac{\tau_{\mathrm{mix}}}{T}}\Big).
\end{equation}
\end{lemma}

The above concentration result applies to any stationary geometrically mixing spike process,
and will be instantiated for both the teacher-side spike proxy (method a)
and the student-side temporal fusion (method b).

\paragraph{Teacher-side spike proxy is consistent.}
In the main text, the spiking map $\sigma_{\text{spike}}$ is implemented by driving a leaky integrate-and-fire (LIF) neuron with a constant input current (BiSpikCLM). 
Accordingly, the teacher spiking attention is \emph{not} an i.i.d.\ Bernoulli process, but a rate-coded spike train generated by LIF dynamics with reset.

Concretely, for each attention entry $(i,j)$, we treat $A_{\mathrm{ANN}}[i,j]$ as a constant input current to a LIF neuron, and define
\begin{equation}
\hat A^{\mathrm{ANN}}_{\mathrm{spike}}(t)[i,j]
= \mathrm{LIF}\big(A_{\mathrm{ANN}}[i,j]\big)_t \in \{0,1\},
\qquad t=1,\ldots,T .
\end{equation}
The resulting spike process is stationary after a short transient and admits a well-defined firing rate
\begin{equation}
p_{ij} = \mathbb{E}\big[\hat A^{\mathrm{ANN}}_{\mathrm{spike}}(t)[i,j]\big]
= g\big(A_{\mathrm{ANN}}[i,j]\big),
\end{equation}
where $g(\cdot)$ is the implicit rate–current transfer function of the LIF neuron, which is monotone increasing.

Although $\{\hat A^{\mathrm{ANN}}_{\mathrm{spike}}(t)[i,j]\}_{t\ge1}$ exhibits temporal correlations due to membrane integration and reset, the leak factor $\beta<1$ induces exponential forgetting of past states. Under standard regularity assumptions for LIF dynamics, the spike process is geometrically $\alpha_{\mathrm{mix}}$-mixing with
\begin{equation}
\alpha_{\mathrm{mix}}(k) \le c_0 \exp(-k/\tau_{\mathrm{mix}}^{\mathrm{ANN}}),
\qquad
\tau_{\mathrm{mix}}^{\mathrm{ANN}} = \mathcal{O}((1-\beta)^{-1}).
\end{equation}

Let
\begin{equation}
\bar A^{\mathrm{ANN}}_{\mathrm{spike}}
= \frac{1}{T}\sum_{t=1}^{T} \hat A^{\mathrm{ANN}}_{\mathrm{spike}}(t).
\end{equation}
Then, by the same mixing-based concentration argument as in Lemma~\ref{lem:temporal_fusion_conc}, we have
\begin{equation}
\mathbb{E}\big[\bar A^{\mathrm{ANN}}_{\mathrm{spike}}\big]
= g(A_{\mathrm{ANN}}), \qquad
\mathbb{E}\!\left[\left\|\bar A^{\mathrm{ANN}}_{\mathrm{spike}} - g(A_{\mathrm{ANN}})\right\|_F^2\right]
= \mathcal{O}\!\Big(\frac{\tau_{\mathrm{mix}}^{\mathrm{ANN}}\,L^2}{T}\Big).
\end{equation}

Therefore, $\bar A^{\mathrm{ANN}}_{\mathrm{spike}}$ can be interpreted as a consistent estimator of the
rate-coded teacher attention $g(A_{\mathrm{ANN}})$ under geometric mixing.
This provides a principled spiking-domain representation of the teacher attention
that is directly comparable to temporally averaged student spiking attention.

\paragraph{Student-side temporal fusion concentrates to the rate-coded attention.}
Fix an input $x$ and an entry $(i,j)$ of the attention matrix. Let
\begin{equation}
X_t = A_{\mathrm{SNN}}(t)[i,j] \in \{0,1\}
\end{equation}
be the spike at time step $t$, and define its conditional mean
\begin{equation}
\mu = \mathbb{E}[X_t \mid x].
\end{equation}
The temporally averaged spiking attention is
\begin{equation}
\bar X = \frac{1}{T} \sum_{t=1}^{T} X_t.
\end{equation}

For LIF-type dynamics, the leak factor $\beta<1$ induces exponential forgetting of past states.
Under mild regularity, the induced spike process is geometrically mixing with
$\tau_{\mathrm{mix}}=\mathcal{O}((1-\beta)^{-1})$, so temporal averaging achieves an effective sample size
$T_{\mathrm{eff}}\approx T/\tau_{\mathrm{mix}}$.

Extending to the full attention matrix, define
\begin{equation}
\bar A_{\mathrm{SNN}} = \frac{1}{T} \sum_{t=1}^T A_{\mathrm{SNN}}(t), \qquad
A^\star = \mathbb{E}[A_{\mathrm{SNN}}(t) \mid x].
\end{equation}

\begin{proposition}[Matrix concentration and loss consistency under mixing]
\label{prop:fusion_rate_match}
Assume each entry process $\{A_{\mathrm{SNN}}(t)[i,j]\}_{t\ge 1}$ satisfies the condition in Lemma~\ref{lem:temporal_fusion_conc} with the same $\tau_{\mathrm{mix}}$ (conditionally on $x$). Then
\begin{equation}
\mathbb{E}\big[\|\bar A_{\mathrm{SNN}} - A^\star\|_F^2 \mid x\big]
= \sum_{i,j}\mathbb{E}\big[(\bar A_{\mathrm{SNN}}[i,j]-A^\star[i,j])^2\mid x\big]
= \mathcal{O}\Big(\frac{\tau_{\mathrm{mix}}\,L^2}{T}\Big).
\end{equation}
Furthermore, for $\hat{\mathcal{L}} = \|A_{\mathrm{ANN}}-\bar A_{\mathrm{SNN}}\|_F^2$ and
$\mathcal{L}^\star = \|A_{\mathrm{ANN}}-A^\star\|_F^2$,
\begin{equation}
\mathbb{E}\big[|\hat{\mathcal{L}}-\mathcal{L}^\star|\mid x\big]
\le 2\|A_{\mathrm{ANN}}-A^\star\|_F\,\mathbb{E}\big[\|\bar A_{\mathrm{SNN}}-A^\star\|_F\mid x\big]
+\mathbb{E}\big[\|\bar A_{\mathrm{SNN}}-A^\star\|_F^2\mid x\big],
\end{equation}
which vanishes as $T\to\infty$ provided $T\gg \tau_{\mathrm{mix}}$ (see, e.g., \citet{tropp2015introduction,levin2017markov}).
\end{proposition}

\paragraph{Alignment of methods (a) and (b).}
By Proposition~\ref{prop:fusion_rate_match}, $\bar A_{\mathrm{SNN}}$ is a consistent estimator of the
student rate-coded attention $A^\star$ when $T\gg\tau_{\mathrm{mix}}$.
Combined with the teacher-side result that $\bar A^{\mathrm{ANN}}_{\mathrm{spike}}$
consistently estimates $g(A_{\mathrm{ANN}})$, the rate-based alignment loss in method~(a),
\begin{equation}
\left\|\bar A^{\mathrm{ANN}}_{\mathrm{spike}} - \bar A_{\mathrm{SNN}}\right\|_F^2,
\end{equation}
can be viewed as an empirical approximation of the ideal population loss
$\|g(A_{\mathrm{ANN}})-A^\star\|_F^2$.
The gap between the two vanishes as
$T\gg \max\{\tau_{\mathrm{mix}},\tau_{\mathrm{mix}}^{\mathrm{ANN}}\}$.

While the above analysis is stated in the asymptotic regime $T\to\infty$,
in practice we employ a small number of spiking time steps (e.g., $T=2$ or $4$).
Importantly, the Rate-MSE objective does not require exact rate convergence,
but only an approximately order-preserving correspondence between input drives
and spike activities, i.e., larger attention values induce higher spike probabilities
even with few time steps under LIF dynamics.
Moreover, stochastic averaging across tokens, attention heads, and training iterations
further mitigates the variance introduced by finite-$T$ sampling.

As a result, the mapping $a \mapsto \hat r_T(a)$ implemented by the LIF-based
spiking map $\sigma_{\mathrm{spike}}$ provides an effective spiking-domain proxy
of ANN attention signals.
Minimizing the Rate-MSE loss in Eq.~\eqref{eq:rate_mse} therefore encourages the
student SNN to match the relative structure of the teacher’s attention,
bridging method~(a) and method~(b) in both theory and practice.

\subsection{Spike-Feature Alignment (SFA)}
\label{appendix:SpAD_feat}

We align intermediate features to transfer compositional and hierarchical knowledge.
Let $H_{\text{ANN}}\in\mathbb{R}^{L\times d_T}$ be the teacher hidden state at a layer and
$H_{\text{SNN}}(t)\in\{0,1\}^{L\times d_S}$ be the student spike feature at time step $t$.
Following prior feature distillation in spiking Transformers (e.g., SpikeBERT and related works),
we map teacher/student features into the same content space:
\begin{equation}
H'_{\text{ANN}} = H_{\text{ANN}},\qquad
H'_{\text{SNN}} = \mathrm{LayerNorm}\!\left(\mathrm{MLP}\!\left(\frac{1}{T}\sum_{t=1}^{T} H_{\text{SNN}}(t)\right)\right),
\label{eq:feat_map}
\end{equation}
where the MLP (or linear projection) matches dimensionality ($d_S\rightarrow d_T$) and LayerNorm stabilizes scale.
When teacher and student have different depths, we align layers using uniform skipping (e.g., every $\lceil B/M\rceil$ layers when $B>M$).
The feature loss combines rate-based and temporally fused MSE:
\begin{equation}
\bar H_{\text{SNN}}=\frac{1}{T}\sum_{t=1}^{T}H_{\text{SNN}}(t),\qquad
\mathcal{L}^{\text{MSE}}_{\text{feat}} = \left\|H'_{\text{ANN}}-H'_{\text{SNN}}\right\|_F^2,
\qquad
\mathcal{L}^{\text{RateMSE}}_{\text{feat}} = \left\|g(H_{\text{ANN}})-\bar H_{\text{SNN}}\right\|_F^2,
\end{equation}
and
\begin{equation}
\mathcal{L}_{\text{feat}} = \gamma_{\text{feat}} \mathcal{L}^{\text{RateMSE}}_{\text{feat}} + (1-\gamma_{\text{feat}})\, \mathcal{L}^{\text{MSE}}_{\text{feat}}.
\end{equation}

\subsection{Soft-Target Alignment (STA)}
\label{appendix:SpAD_soft}

We distill output distributions using softened logits:
\begin{equation}
\mathcal{L}_{\text{soft}}
=
\tau^2\,\mathrm{KL}\!\left(
\mathrm{softmax}\!\left(\frac{z_{\text{ANN}}}{\tau}\right)\ \Big\|\ 
\mathrm{softmax}\!\left(\frac{z_{\text{SNN}}}{\tau}\right)
\right),
\end{equation}
where $z_{\text{ANN}}$ and $z_{\text{SNN}}$ are teacher/student logits and $\tau$ is the temperature.

\subsection{Hard-Target Alignment (HTA)}
\label{appendix:SpAD_hard}

We include standard cross-entropy with the ground-truth next token $y$:
\begin{equation}
\mathcal{L}_{\text{hard}}
=
\mathrm{CE}\!\left(\mathrm{softmax}(z_{\text{SNN}}),\, y\right).
\end{equation}

\subsection{Total Loss}
\label{appendix:SpAD_total}

\begin{equation}
\mathcal{L}_{\text{total}}
=
\lambda_1 \mathcal{L}_{\text{emb}}
+
\lambda_2 \mathcal{L}_{\text{attn}}
+
\lambda_3 \mathcal{L}_{\text{feat}}
+
\lambda_4 \mathcal{L}_{\text{soft}}
+
\lambda_5 \mathcal{L}_{\text{hard}}.
\end{equation}
Each $\lambda_i$ balances the corresponding objective. The coefficients $\{\lambda_i\}_{i=1}^{5}$ are listed in Table~\ref{tab:training_hyperparams} and satisfy $\sum_{i=1}^{5}\lambda_i=1$ for a stable loss scale; we set them based on preliminary sweeps, with larger weights on distillation terms improving convergence and accuracy.

\section{Training Details}
\label{appendix:training}

\begin{table}[h]
\caption{Summary of training hyperparameters and configurations used for BiSpikCLM, including optimization settings, distillation parameters, and hardware specifications.}
\centering
\footnotesize
\label{tab:training_hyperparams}
\begin{tabular}{ll}
\toprule
\textbf{Hyperparameter}         & \textbf{Value} \\
\midrule
Teacher ANN model                   & OPT-family \\
Student SNN model                   & BiSpikCLM \\
Tokenizer / Vocabulary          & Aligned with OPT \\
Batch size (per GPU)            & 16 \\
Gradient accumulation steps     & 16 \\
Effective batch size            & 256 \\
Optimizer                       & Adam \\
Learning rate                   & $5 \times 10^{-4}$ \\
Scheduler                       & Cosine decay \\
Warm-up ratio                   & 0.2 \\
Gradient clipping threshold     & 0.7 \\
Temperature $\tau$ (for SpAD)    & 2.0 \\
Distillation weights ($\lambda_1$ to $\lambda_5$) & 0.2, 0.1, 0.1, 0.3, 0.3 \\
Inference time steps ($T$)      & 2 and 4 \\
Hardware                        & $8 \times$ NVIDIA RTX 4090 (24GB) \\
\bottomrule
\end{tabular}
\end{table}

\paragraph{Training Paradigm}
The training of \textbf{BiSpikCLM} follows a teacher–student paradigm, where the teacher model is a pre-trained open-source ANN-based large language model from the OPT family, and the student is our spike-based BiSpikCLM. To ensure consistency between the teacher and student models, we align both the vocabulary and tokenizer with those used in OPT.

\paragraph{Optimization Setup} 
All experiments are conducted using a batch size of 16 and a gradient accumulation factor of 16, effectively yielding a total batch size of 256 tokens. The optimizer used is Adam, with a learning rate set to $5 \times 10^{-4}$. A cosine learning rate scheduler with a warm-up ratio of 0.2 is employed to stabilize early training. Gradient clipping is applied with a threshold of 0.7 to avoid exploding gradients in the early training phases, which can be particularly pronounced in spiking models. All models are trained on 8 NVIDIA RTX 4090 GPUs (24GB each).

\paragraph{Spike-Aware Alignment Distillation} 
For Spike-Aware Alignment Distillation (SpAD), we adopt a temperature of $\tau = 2.0$ for the teacher’s soft targets. The overall loss is computed as a weighted combination of multiple alignment objectives defined in Method Section. The corresponding loss weights are: $\lambda_1 = 0.2$, $\lambda_2 = 0.1$, $\lambda_3 = 0.1$, $\lambda_4 = 0.3$, and $\lambda_5 = 0.3$. These coefficients are summarized in Table~\ref{tab:training_hyperparams}. We set $\sum_{i=1}^{5}\lambda_i = 1$ to keep the overall loss scale stable across experiments. The weights are selected based on preliminary sweeps: placing more emphasis on token-level attention/output alignment improves convergence and accuracy, whereas assigning too much weight to auxiliary terms can hurt training stability.

\paragraph{Inference Time Steps} 
To investigate the trade-off between accuracy and energy efficiency, BiSpikCLM is trained with varying numbers of inference time steps, specifically $T=2$ and $T=4$. A higher number of steps improves temporal resolution and accuracy at the cost of increased energy consumption, enabling flexible deployment depending on the application constraints.

\paragraph{Training Data Selection} 
The subset of training data used in our experiments was drawn from the fineweb-edu dataset~\citep{penedo2024fineweb}, specifically the 10BT sample accessible at \url{https://huggingface.co/datasets/HuggingFaceFW/fineweb-edu/tree/main/sample/10BT}. Since our claim of using a lower training token volume is a core contribution, it is crucial to detail how the data was sampled to ensure reproducibility. This subset was selected directly from the publicly available sample without additional filtering or preprocessing, providing other researchers with a clear and reproducible training set.

\section{Surrogate Gradient and Backpropagation Through Time in SNNs}
\label{appendix:surrogate}

\paragraph{Surrogate gradient.}
Training spiking neural networks (SNNs) presents a significant challenge due to the non-differentiable nature of spike generation functions, such as the Heaviside step function used in the spiking neuron model. To enable end-to-end optimization with backpropagation, we adopt a surrogate gradient approach introduced by ~\citet{fang2023spikingjelly}.

Specifically, the discrete spiking activation $S$ is approximated by a continuous and differentiable function using an arctangent-based surrogate:
\begin{equation}
S(U) \approx \frac{1}{\pi}\arctan\!\left(\frac{\pi}{2}\alpha U\right) + \frac{1}{2},
\label{eq:surrogate}
\end{equation}
where $U$ is the membrane potential and $\alpha$ is a tunable hyperparameter controlling the sharpness of the transition. In our experiments, we set $\alpha = 2$ by default, balancing gradient magnitude and smoothness.

Taking the derivative of Equation~\eqref{eq:surrogate} yields the surrogate gradient used during backpropagation:
\begin{equation}
\frac{\partial S}{\partial U}
= \frac{\alpha}{2}\cdot\frac{1}{1 + \left(\frac{\pi}{2}\alpha U\right)^2}.
\label{eq:grad}
\end{equation}
This surrogate formulation enables stable and effective gradient-based optimization for BiSpikCLM. It allows error signals to be backpropagated through spike-generating layers without requiring exact gradients, thus making the training pipeline compatible with standard deep learning frameworks.

\paragraph{Boundedness.}
A useful property of the arctangent surrogate in Eq.~\eqref{eq:grad} is that the gradient is bounded:
\begin{equation}
0 \le \frac{\partial S}{\partial U} \le \frac{\alpha}{2}.
\label{eq:surrogate_bound}
\end{equation}
This bounded surrogate derivative mitigates gradient explosion originating from spike nonlinearity and contributes to stable optimization in practice.

\subsection{Backpropagation Through Time}
\label{appendix:bptt_snn}

We further provide the gradient derivation of spiking neurons trained with backpropagation through time (BPTT), following the standard formulation in spiking deep networks.
Let the total loss be a sum of per time step losses (e.g., the losses defined in ~\eqref{Loss} in the main paper):
\begin{equation}
\mathcal{L} = \sum_{t=1}^{T} \mathcal{L}_t.
\end{equation}
Since the synaptic weights are shared across time, the global gradient is
\begin{equation}
\frac{\partial \mathcal{L}}{\partial W} = \sum_{t=1}^{T} \frac{\partial \mathcal{L}_t}{\partial W}.
\label{eq:global_grad}
\end{equation}

\paragraph{Cross-time dependency.}
Although~\eqref{eq:global_grad} sums per time step contributions, each $\mathcal{L}_t$ depends on the shared weight $W$
\emph{both directly} through the instantaneous synaptic term $W X_t$ and \emph{indirectly} through the recurrent state
propagation $U_{t}\!\leftarrow\!U_{t-1}\!\leftarrow\!\cdots$.
To make this dependence explicit, we can write
\begin{equation}
\frac{\partial \mathcal{L}}{\partial W}
= \sum_{t=1}^{T} \sum_{j=1}^{t} \frac{\partial \mathcal{L}_t}{\partial W_j}\frac{\partial W_j}{\partial W},
\label{eq:cross_time_sum}
\end{equation}
where $W_j$ denotes the (conceptual) copy of the shared parameter applied at time step $j$.
Since weights are shared across time, $W_1=\cdots=W_T=W$, hence $\partial W_j/\partial W = 1$ and
\begin{equation}
\frac{\partial \mathcal{L}}{\partial W}
= \sum_{t=1}^{T} \sum_{j=1}^{t} \frac{\partial \mathcal{L}_t}{\partial W_j}.
\label{eq:cross_time_shared}
\end{equation}
The inner sum captures the cross-time influence of $W$ on $\mathcal{L}_t$ via all earlier states, which will be compactly represented by the eligibility trace recursion below.

\paragraph{Neuron dynamics.}
Consider a discrete-time spiking neuron with membrane potential $U_t$ and spike output $S_t$.
A commonly used update takes the form
\begin{equation}
U_t = \beta U_{t-1} + W X_t,
\label{eq:mem_update_simple}
\end{equation}
where $X_t$ denotes the pre-synaptic input at time step $t$, and $\beta \in (0,1)$ is the decay factor.
(If a reset term is used, see the remark below.)
The spike is generated by a thresholding function $S_t = \mathbf{1}[U_t > \theta]$, whose derivative is approximated by the surrogate gradient in Eq.~\eqref{eq:grad}.

\paragraph{Local error term.}
We derive the gradient of the per time step loss with respect to the shared synaptic weights.
By the chain rule, for each time step $t$,
\begin{equation}
\frac{\partial \mathcal{L}_t}{\partial W}
=
\frac{\partial \mathcal{L}_t}{\partial S_t}
\cdot
\frac{\partial S_t}{\partial U_t}
\cdot
\frac{\partial U_t}{\partial W}.
\label{eq:chain_rule_full}
\end{equation}
The spiking nonlinearity operates independently on each neuron; therefore,
its Jacobian $\frac{\partial S_t}{\partial U_t}$ is diagonal.
In practice, this Jacobian-vector product reduces to an element-wise (Hadamard) product.
Accordingly, we define the local error term as
\begin{equation}
\delta_t
=
\frac{\partial \mathcal{L}_t}{\partial S_t}
\odot
\frac{\partial S_t}{\partial U_t},
\label{eq:delta_t}
\end{equation}
where $\odot$ denotes element-wise multiplication and
$\frac{\partial S_t}{\partial U_t}$ is given by the surrogate gradient in Eq.~\eqref{eq:grad}.

With this definition, the gradient with respect to the shared weights can be compactly written as
\begin{equation}
\frac{\partial \mathcal{L}}{\partial W}
=
\sum_{t=1}^{T}
\delta_t \cdot \frac{\partial U_t}{\partial W}.
\label{eq:grad_sum_delta}
\end{equation}

\paragraph{Eligibility trace recursion.}
A key step in BPTT for SNNs is that $\frac{\partial U_t}{\partial W}$ admits an efficient recurrence.
From Eq.~\eqref{eq:mem_update_simple}, we have
\begin{equation}
\frac{\partial U_t}{\partial W} = X_t + \beta \frac{\partial U_{t-1}}{\partial W}.
\end{equation}
Define the eligibility trace $e_t = \frac{\partial U_t}{\partial W}$ with $e_0=0$.
Then
\begin{equation}
e_t = X_t + \beta e_{t-1},
\qquad
\frac{\partial \mathcal{L}}{\partial W} = \sum_{t=1}^{T} \delta_t\, e_t .
\label{eq:eligibility_trace}
\end{equation}
This form makes explicit that the gradient is the accumulation over time of a local error term $\delta_t$ weighted by the eligibility trace $e_t$.

\subsection{Stability of the eligibility trace}
\label{appendix:etrace_bound}

We analyze the stability of the eligibility trace in Eq.~\eqref{eq:eligibility_trace}.
Consider the recurrence
\begin{equation}
e_t = X_t + \beta e_{t-1}, \qquad t=1,\ldots,T,
\label{eq:etrace_rec}
\end{equation}
where $\beta\in(0,1)$ and $e_0=0$.
Let $\|\cdot\|$ denote any norm that satisfies (i) the triangle inequality and (ii) positive homogeneity
(e.g., $\ell_2$ norm for vectors or Frobenius norm for matrices).
Assume the pre-synaptic input is uniformly bounded, i.e., $\|X_t\|\le M$ for all $t$.

\begin{lemma}[Uniform bound on $e_t$]
\label{lem:etrace_bound}
Under the above assumptions and $\beta\in(0,1)$, the eligibility trace satisfies
\begin{equation}
\|e_t\| \le \frac{M}{1-\beta}, \qquad \forall t\ge 1.
\label{eq:e_bound_icml}
\end{equation}
\end{lemma}

\begin{proof}
Unrolling the recursion in~\eqref{eq:eligibility_trace} gives
\begin{equation}
e_t = \sum_{k=0}^{t-1}\beta^k X_{t-k}.
\label{eq:etrace_unroll}
\end{equation}
Taking norms on both sides and applying the triangle inequality gives
\begin{equation}
\|e_t\| = \left\|\sum_{k=0}^{t-1}\beta^k X_{t-k}\right\|
\le \sum_{k=0}^{t-1}\left\|\beta^k X_{t-k}\right\|.
\end{equation}
By positive homogeneity of the norm and $\beta^k\ge 0$, we have
\begin{equation}
\sum_{k=0}^{t-1}\left\|\beta^k X_{t-k}\right\|
= \sum_{k=0}^{t-1}\beta^k \|X_{t-k}\|
\le M\sum_{k=0}^{t-1}\beta^k,
\end{equation}
where the last inequality uses the bounded-input assumption $\|X_{t-k}\|\le M$.
Finally, since $\beta\in(0,1)$, the geometric series satisfies
\begin{equation}
\sum_{k=0}^{t-1}\beta^k = \frac{1-\beta^t}{1-\beta} \le \frac{1}{1-\beta},
\end{equation}
which concludes the proof of Eq.~\eqref{eq:e_bound_icml}.
\end{proof}

\paragraph{Implication for gradient stability.}
Using~\eqref{eq:surrogate_bound}, we have $\|\partial S_t/\partial U_t\|_\infty \le \alpha/2$.
For any vector norm compatible with element-wise multiplication,
\begin{equation}
\|\delta_t\|
= \left\|\frac{\partial \mathcal{L}_t}{\partial S_t}\odot\frac{\partial S_t}{\partial U_t}\right\|
\le \left\|\frac{\partial \mathcal{L}_t}{\partial S_t}\right\|\left\|\frac{\partial S_t}{\partial U_t}\right\|_\infty
\le \frac{\alpha}{2}\left\|\frac{\partial \mathcal{L}_t}{\partial S_t}\right\|.
\end{equation}
Combining this with Lemma~\ref{lem:etrace_bound} yields the bound
\begin{equation}
\|\delta_t e_t\|
\le \|\delta_t\|\,\|e_t\|
\le \frac{\alpha M}{2(1-\beta)}
\left\|\frac{\partial \mathcal{L}_t}{\partial S_t}\right\|.
\label{eq:grad_term_bound}
\end{equation}
Therefore, under bounded inputs, the temporal accumulation through the eligibility trace does not introduce gradient explosion by itself; its magnitude is controlled by $\beta$ and the surrogate slope parameter $\alpha$.

\paragraph{Remark (with reset).}
If a soft reset is used, e.g.,
\begin{equation}
U_t=\beta U_{t-1}(1-S_{t-1}) + W X_t,
\end{equation}
then the eligibility trace obeys
\begin{equation}
e_t = X_t + \beta(1-S_{t-1})e_{t-1}.
\end{equation}
Since $0\le 1-S_{t-1}\le 1$, the effective contraction factor is at most $\beta$ at each time step, yielding an equal or tighter bound than Lemma~\ref{lem:etrace_bound}.

\clearpage
\section{Theoretical Synaptic Operation and Energy Consumption Calculation}
\label{appendix:energy}
The theoretical energy consumption of \textit{BiSpikCLM} is estimated by first calculating the synaptic operations (SOPs). For each block or layer $l$, we have:
\begin{equation}
    \text{SOPs}(l) = f_r(l) \times T \times \text{FLOPs}(l),
    \label{eq:sops}
\end{equation}
where $l$ indexes a block in \textit{BiSpikCLM}, $f_r(l)$ is the average firing rate of the input spike train to block $l$ (measured as spikes per neuron per time step), and $T$ is the simulation time steps of the spiking neuron. $\text{FLOPs}(l)$ denotes the number of multiply-and-accumulate (MAC) operations of block $l$ in the equivalent ANN. $\text{SOPs}(l)$ thus represents the spike-based accumulate (AC) operations performed in the SNN.

Following \cite{horowitz2014energy}, we assume the energy per operation on a 45\,nm process as
\[
    E_{\text{MAC}} = 4.6 \,\text{pJ}, \qquad E_{\text{AC}} = 0.9 \,\text{pJ}.
\]

For ANNs, the theoretical energy consumption of a block $b$ is
\begin{equation}
    \text{Power}_{\text{ANN}}(b) = E_{\text{MAC}} \times \text{FLOPs}(b).
    \label{eq:annpower}
\end{equation}

For SNNs, the theoretical energy consumption of block $b$ is
\begin{equation}
    \text{Power}_{\text{SNN}}(b) = E_{\text{AC}} \times \text{SOPs}(b).
    \label{eq:snnpower}
\end{equation}

According to (~\citep{horowitz2014energy, kundu2021spike,kundu2021hire,hu2021advancing,yin2021accurate,kim2021optimizing,yao2021temporal}), the total energy consumption of \textit{BiSpikCLM} can be decomposed into three parts: (1) the embedding stage, which is executed with dense MAC operations, (2) the $L$ stacked transformer blocks, each of which is spiking and therefore counted using AC operations, and (3) the language-model head (LM-head) that maps hidden states to vocabulary logits (dense MACs). We write:
\begin{align}
    E_{\text{BiSpikCLM}} &= 
    E_{\text{MAC}} \cdot \big(\text{FLOPs}_{\text{ Embed}} + \text{FLOPs}_{\text{ LM-head}}\big) \notag \\
    &\quad + E_{\text{AC}} \cdot \sum_{l=1}^{L}
    \Big( \text{SOP}_{\text{SFSA}}(l) + \text{SOP}_{\text{SFFN}}(l) \Big)
    \label{eq:total_energy_with_lmhead}
\end{align}

where $\text{FLOPs}_{\text{ Embed}}$ and $\text{FLOPs}_{\text{ LM-head}}$ denote the MAC operations of the embedding stage and the output projection to vocabulary logits, respectively; $\text{SOP}_{\text{SFSA}}(l)$ and $\text{SOP}_{\text{SFFN}}(l)$ represent the spike-accumulate operations of the Spiking Causal Self-Attention and Spiking Feed-Forward Network modules in block $l$; $E_{\text{MAC}}$ and $E_{\text{AC}}$ are the energy costs per MAC and AC operation; $L$ is the number of transformer blocks; and $f_r$ and $T$ denote the average firing rate and the number of simulation time steps.

\clearpage
\section{Comparison with SpikeLLM}
\label{appendix:spikellm}

\begin{table*}[t]
\caption{Comparison of BiSpikCLM and SpikeLLM across different model scales, neuron/spike formats, and time steps. Avg. Acc. reports zero-shot accuracy (\%), and SNN/ANN Ratio shows the performance of spiking models relative to their ANN counterparts.}
\centering
\footnotesize
\label{tab:comparison_spikellm}
\setlength{\tabcolsep}{2.25pt} 
\renewcommand{\arraystretch}{1.05} 
\begin{tabular}{l | c c c c c c c c | c c}
\toprule
\textbf{Model} & 
{ \textbf{\makecell{Params (B)}}} & 
{ \textbf{\makecell{Tokens (B)}}} & 
{ \textbf{\makecell{Spike Form}}} & 
{ \textbf{\makecell{Time Step}}} & 
{\textbf{\makecell{Avg. Acc.(\%) $\uparrow$}}} &
{\textbf{\makecell{SNN/ANN \\ Ratio (\%) $\uparrow$}}} \\
\midrule
BiSpikCLM  & 0.125 & \textbf{1.0}  & Binary & 2 & 36.05 & 93.39 \\
TriSpikCLM  & 0.125 & \textbf{1.0}  & Ternary & 2 & 36.11 & 93.55 \\
BiSpikCLM  & 0.125 & \textbf{1.0}  & Binary & 4 & 36.50 & 94.56 \\
TriSpikCLM  & 0.125 & \textbf{1.0}  & Ternary & 4 & 36.27 & 93.96 \\
\midrule
BiSpikCLM  & 0.350 & \textbf{2.0}  & Binary& 2 & 38.48 & 90.16 \\
TriSpikCLM  & 0.350 & \textbf{2.0}  & Ternary & 2 & 38.14 & 89.36 \\
BiSpikCLM  & 0.350 & \textbf{2.0}  & Binary & 4 & 38.84 & 91.00 \\
TriSpikCLM  & 0.350 & \textbf{2.0}  & Ternary & 4 & 38.87 & 91.07 \\
\midrule
BiSpikCLM  & 1.300 & \textbf{10.0}  & Binary & 2 & 41.33 & 83.11 \\
TriSpikCLM  & 1.300 & \textbf{10.0}  & Ternary & 2 & 41.08 & 82.60 \\
BiSpikCLM  & 1.300 & \textbf{10.0}  & Binary & 4 & 42.19 & 84.84 \\
TriSpikCLM  & 1.300 & \textbf{10.0}  & Ternary & 4 & 42.12 & 84.70 \\
\midrule
SpikeLLM & 7.000 & — & Integer (W2A16) & 2 &  49.92 & 78.17 \\
SpikeLLM & 7.000 & — & Integer  (W2A8) & 4 &  41.77 & 65.41   \\
SpikeLLM & 13.00 & — & Integer (W2A16) & 2 &  53.76 & 81.34 \\
SpikeLLM & 13.00 & — & Integer  (W2A8) & 4 &  50.12 & 75.78 \\
SpikeLLM & 70.00 & — & Integer (W2A16) & 2 &  60.47 & 82.55 \\
\bottomrule
\end{tabular}
\end{table*}

Table~\ref{tab:comparison_spikellm} compares BiSpikCLM with SpikeLLM across different model scales, quantization methods, spike forms, and simulation time steps. Overall, BiSpikCLM achieves competitive zero-shot accuracy with smaller models and lower-precision spike forms. Notably, the SNN/ANN ratio of BiSpikCLM is consistently higher (82–95\%) than that of SpikeLLM, indicating that our spike-based models retain more of the original ANN performance. This improvement is largely attributed to our spike-aware knowledge distillation framework, which effectively transfers information from ANN teachers to spiking students. These results not only highlight the efficiency and effectiveness of our approach but also provide a promising pathway for further improving SNN-LLMs through advanced distillation strategies.

\section{Weight Visualization}
\label{appendix:visual_weight}

In this subsection, we analyze the weight distributions of Artificial Neural Networks (ANNs) and Spiking Neural Networks (SNNs) across different layers and components. The weight visualization provides valuable insights into the fundamental differences between these two types of neural networks and highlights the unique characteristics of SNNs. The weight distributions of ANNs and SNNs exhibit distinct patterns across various layers and components (\texttt{q\_proj}, \texttt{k\_proj}, \texttt{v\_proj}, \texttt{out\_proj}, \texttt{fc1}, \texttt{fc2}).

Our approach differs from traditional ANN-to-SNN conversion methods in that we do not passively fit SNN weights to match those of ANNs. Instead, we actively capture and adapt to the unique characteristics of SNNs. This active adaptation is crucial for leveraging the full potential of SNNs, which operate on spiking dynamics rather than continuous activation values. By focusing on the inherent properties of SNNs, such as their broader weight distribution and dynamic spiking behavior, our method ensures that the network is optimized for spiking neural computation. This approach allows SNNs to maintain their distinct advantages, such as energy efficiency and biological plausibility, while still achieving high performance.

In summary, the weight visualization clearly demonstrates the differences between ANNs and SNNs. Our method capitalizes on these differences by actively adapting to the unique characteristics of SNNs, rather than forcing them to conform to the weight and activation patterns of ANNs. This approach is essential for developing effective and efficient SNNs that can fully leverage their spiking dynamics.

\begin{figure}[H]
    \centering
    \includegraphics[width=0.95\textwidth]{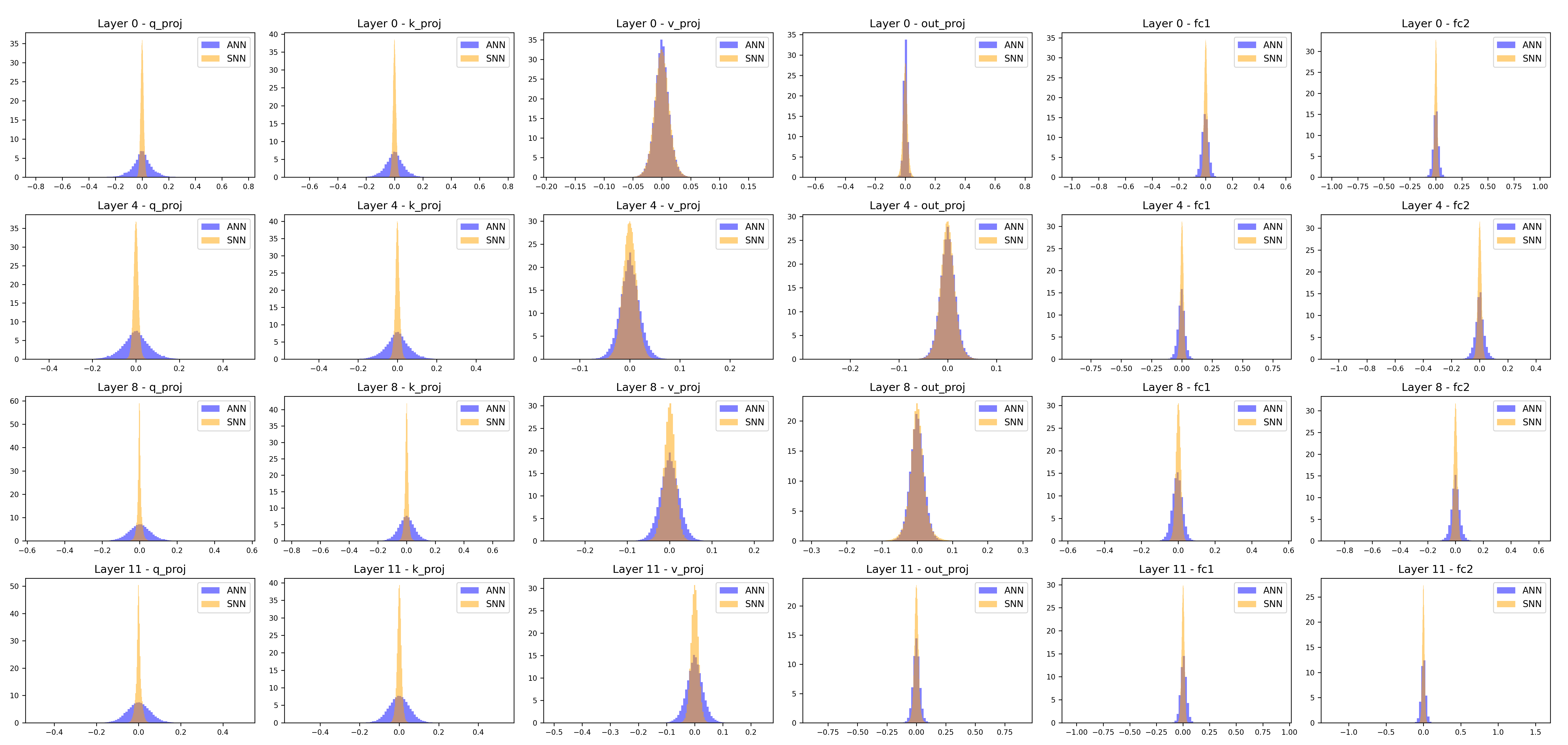}
    \caption{Weight distribution comparison between Artificial Neural Networks (ANNs) and Spiking Neural Networks (SNNs). 
ANNs typically exhibit a more concentrated weight distribution around zero, especially in early layers (e.g., Layer 0). 
In deeper layers (e.g., Layer 11), their weight distribution becomes slightly more spread out but remains relatively compact, indicating tightly clustered weights that contribute to stability and ease of training. 
In contrast, SNNs show a broader and more dispersed weight distribution, with weights less tightly clustered around zero. 
This broader spread is particularly notable in deeper layers (e.g., Layer 11), reflecting the dynamic and diverse weight updates characteristic of their spiking nature.
}
    \label{fig:weight}
\end{figure}

\section{Outlier Analysis of Weight Distributions}
\label{appendix:outlier_analysis}

To further investigate the differences between Artificial Neural Networks (ANNs) and Spiking Neural Networks (SNNs), we analyze the number of outliers in the weight distributions across various layers and components. Outliers are defined as weights that significantly deviate from the mean, potentially indicating instability or over-parameterization in specific components.

\subsection{Total Number of Outliers}

As summarized in Table~\ref{tab:outlier_summary}, ANNs exhibit significantly more outliers, with over three times as many compared to SNNs across all evaluated layers.

\begin{table}[H]
\caption{Total number of weight outliers across all layers and components.}
\centering
\label{tab:outlier_summary}
\begin{tabular}{lcc}
\toprule
\textbf{Model} & \textbf{Total Outliers} \\
\midrule
ANN & 18,663 \\
SNN & 5,834 \\
\bottomrule
\end{tabular}
\end{table}

\subsection{Layer-wise Comparison}
To gain deeper insights, we present a detailed layer-wise and component-wise comparison in Table~\ref{tab:ann_snn_weights}, showing the number of outliers for both models. \textbf{ANNs exhibit a significantly higher number of outliers}, especially in deeper fully connected layers (\texttt{fc1}, \texttt{fc2}), which may be due to larger weight magnitudes and higher variance. While \textbf{SNNs show fewer outliers overall}, reflecting their more compact and tightly regulated weight distributions. Interestingly, in some components (e.g., \texttt{q\_proj}, \texttt{k\_proj} at Layer 7), SNNs have more outliers than ANNs. This suggests local spikes in weight variability, possibly due to the intrinsic dynamics of spiking updates. The standard deviation of weights (not shown here) is consistently lower in SNNs, reinforcing the observation that they operate within a narrower, more stable range. These findings highlight a fundamental difference in the behavior of ANNs and SNNs: while ANNs may rely on larger weight magnitudes and are more prone to extreme values, SNNs exhibit smoother, biologically plausible weight distributions that reduce the risk of instability.

\begin{table}[H]
\caption{Comparison of weight statistics between ANN and SNN across various layers and components.}
\label{tab:ann_snn_weights}
\centering
\footnotesize
\begin{tabular}{cccccccc}
\toprule
\textbf{Layer} & \textbf{Component} & \textbf{Model} & \textbf{Mean} & \textbf{Std} & \textbf{Max} & \textbf{Min} & \textbf{Num\_Outliers} \\
\midrule
0  & q\_proj   & ANN & -3.15E-05 & 7.87E-02 & 0.77 & -0.80 & 596 \\
0  & q\_proj   & SNN & 1.30E-04  & 1.14E-02 & 0.06 & -0.06 & \textbf{6} \\
0  & k\_proj   & ANN & 6.12E-05  & 7.26E-02 & 0.77 & -0.70 & 1154 \\
0  & k\_proj   & SNN & -3.36E-05 & 1.10E-02 & 0.06 & -0.06 & \textbf{24} \\
0  & v\_proj   & ANN & -1.55E-05 & 1.28E-02 & 0.17 & -0.19 & 163 \\
0  & v\_proj   & SNN & 2.97E-05  & 1.31E-02 & 0.08 & -0.08 & \textbf{34} \\
0  & out\_proj & ANN & -8.64E-06 & 1.30E-02 & 0.78 & -0.62 & 702 \\
0  & out\_proj & SNN & -2.16E-04 & 1.58E-02 & 0.12 & -0.11 & \textbf{73} \\
0  & fc1       & ANN & -3.17E-03 & 2.93E-02 & 0.56 & -1.00 & 5657 \\
0  & fc1       & SNN & 4.58E-04  & 1.16E-02 & 0.06 & -0.06 & \textbf{2} \\
0  & fc2       & ANN & -1.23E-05 & 2.62E-02 & 1.00 & -1.00 & 716 \\
0  & fc2       & SNN & 2.74E-04  & 1.24E-02 & 0.06 & -0.11 & \textbf{189} \\
3  & q\_proj   & ANN & -2.93E-04 & 6.02E-02 & 0.56 & -0.48 & 105 \\
3  & q\_proj   & SNN & 1.98E-04  & 1.15E-02 & 0.07 & -0.10 & \textbf{57} \\
3  & k\_proj   & ANN & 1.19E-04  & 6.16E-02 & 0.45 & -0.46 & 116 \\
3  & k\_proj   & SNN & 1.34E-04  & 1.10E-02 & 0.10 & -0.10 & \textbf{199} \\
3  & v\_proj   & ANN & 1.71E-05  & 2.06E-02 & 0.23 & -0.27 & 60  \\
3  & v\_proj   & SNN & -2.36E-04 & 1.39E-02 & 0.09 & -0.09 & \textbf{63} \\
3  & out\_proj & ANN & 2.38E-06  & 1.74E-02 & 0.50 & -0.36 & 183 \\
3  & out\_proj & SNN & -2.75E-04 & 1.48E-02 & 0.10 & -0.10 & \textbf{79} \\
3  & fc1       & ANN & -1.81E-03 & 2.49E-02 & 0.86 & -0.73 & 2045 \\
3  & fc1       & SNN & 2.39E-05  & 1.30E-02 & 0.07 & -0.08 & \textbf{38} \\
3  & fc2       & ANN & -2.44E-06 & 2.67E-02 & 0.64 & -1.06 & 2053 \\
3  & fc2       & SNN & 4.87E-05  & 1.30E-02 & 0.08 & -0.10 & \textbf{137} \\
7  & q\_proj   & ANN & -1.80E-04 & 5.98E-02 & 0.53 & -0.54 & 147 \\
7  & q\_proj   & SNN & -4.13E-04 & 1.10E-02 & 0.13 & -0.13 & \textbf{1484} \\
7  & k\_proj   & ANN & -2.34E-05 & 6.03E-02 & 0.74 & -0.68 & 185 \\
7  & k\_proj   & SNN & 5.25E-04  & 1.35E-02 & 0.17 & -0.19 & \textbf{1004} \\
7  & v\_proj   & ANN & 1.22E-05  & 2.02E-02 & 0.13 & -0.14 & 20  \\
7  & v\_proj   & SNN & -2.11E-04 & 1.35E-02 & 0.12 & -0.12 & \textbf{207} \\
7  & out\_proj & ANN & 1.25E-05  & 1.76E-02 & 0.19 & -0.19 & 93  \\
7  & out\_proj & SNN & -1.91E-04 & 2.03E-02 & 0.19 & -0.17 & \textbf{609} \\
7  & fc1       & ANN & -4.40E-03 & 2.73E-02 & 0.59 & -0.59 & 315 \\
7  & fc1       & SNN & 6.75E-05  & 1.30E-02 & 0.15 & -0.13 & \textbf{224} \\
7  & fc2       & ANN & -3.44E-05 & 3.00E-02 & 0.38 & -1.00 & 1315 \\
7  & fc2       & SNN & -8.44E-05 & 1.24E-02 & 0.10 & -0.12 & \textbf{204} \\
11 & q\_proj   & ANN & -2.52E-04 & 5.57E-02 & 0.50 & -0.49 & 153 \\
11 & q\_proj   & SNN & -1.89E-04 & 1.19E-02 & 0.15 & -0.12 & \textbf{181} \\
11 & k\_proj   & ANN & 5.30E-05  & 5.44E-02 & 0.53 & -0.53 & 366 \\
11 & k\_proj   & SNN & 2.98E-04  & 1.17E-02 & 0.19 & -0.11 & \textbf{225} \\
11 & v\_proj   & ANN & 4.60E-05  & 2.95E-02 & 0.20 & -0.50 & 68  \\
11 & v\_proj   & SNN & -5.92E-05 & 1.35E-02 & 0.24 & -0.19 & \textbf{268} \\
11 & out\_proj & ANN & -2.00E-05 & 3.15E-02 & 0.89 & -0.88 & 478 \\
11 & out\_proj & SNN & -1.07E-05 & 1.96E-02 & 0.19 & -0.16 & \textbf{325} \\
11 & fc1       & ANN & 4.87E-03  & 2.82E-02 & 0.92 & -1.04 & 473 \\
11 & fc1       & SNN & -4.00E-05 & 1.37E-02 & 0.07 & -0.09 & \textbf{52} \\
11 & fc2       & ANN & 3.81E-05  & 3.23E-02 & 1.56 & -1.23 & 1500 \\
11 & fc2       & SNN & 2.15E-05  & 1.50E-02 & 0.12 & -0.11 & \textbf{150} \\
\bottomrule
\end{tabular}
\end{table}

\clearpage
\section{Firing Visualization}
\label{appendix:visual_firing}
\begin{figure}[!ht]
    \centering
    \begin{subfigure}{0.42\textwidth}
        \centering
        \includegraphics[width=\linewidth]{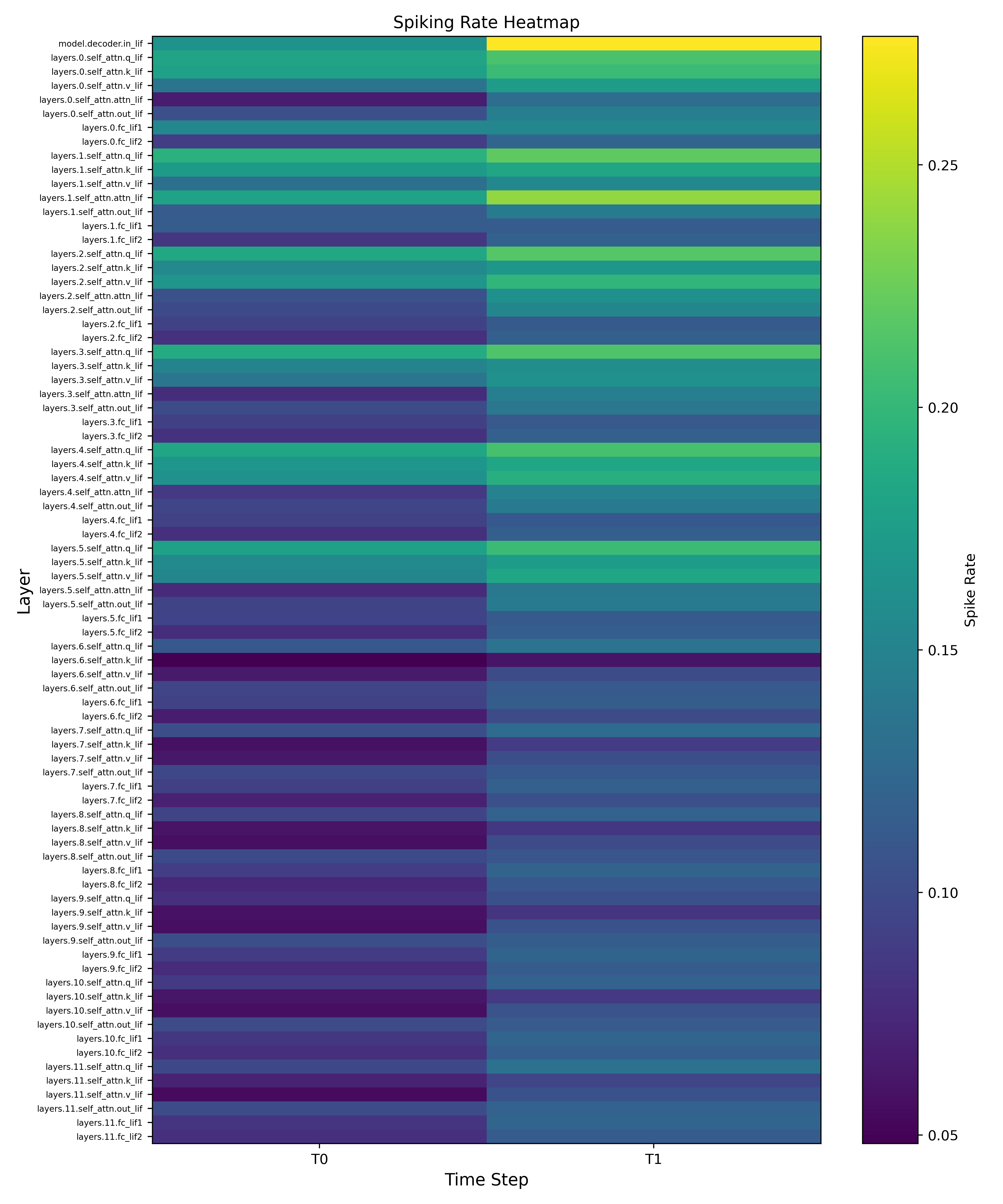}
        \caption{$T=2$}
        \label{fig:heatmap_t2}
    \end{subfigure}
    \hspace{0.02\textwidth}
    \begin{subfigure}{0.42\textwidth}
        \centering
        \includegraphics[width=\linewidth]{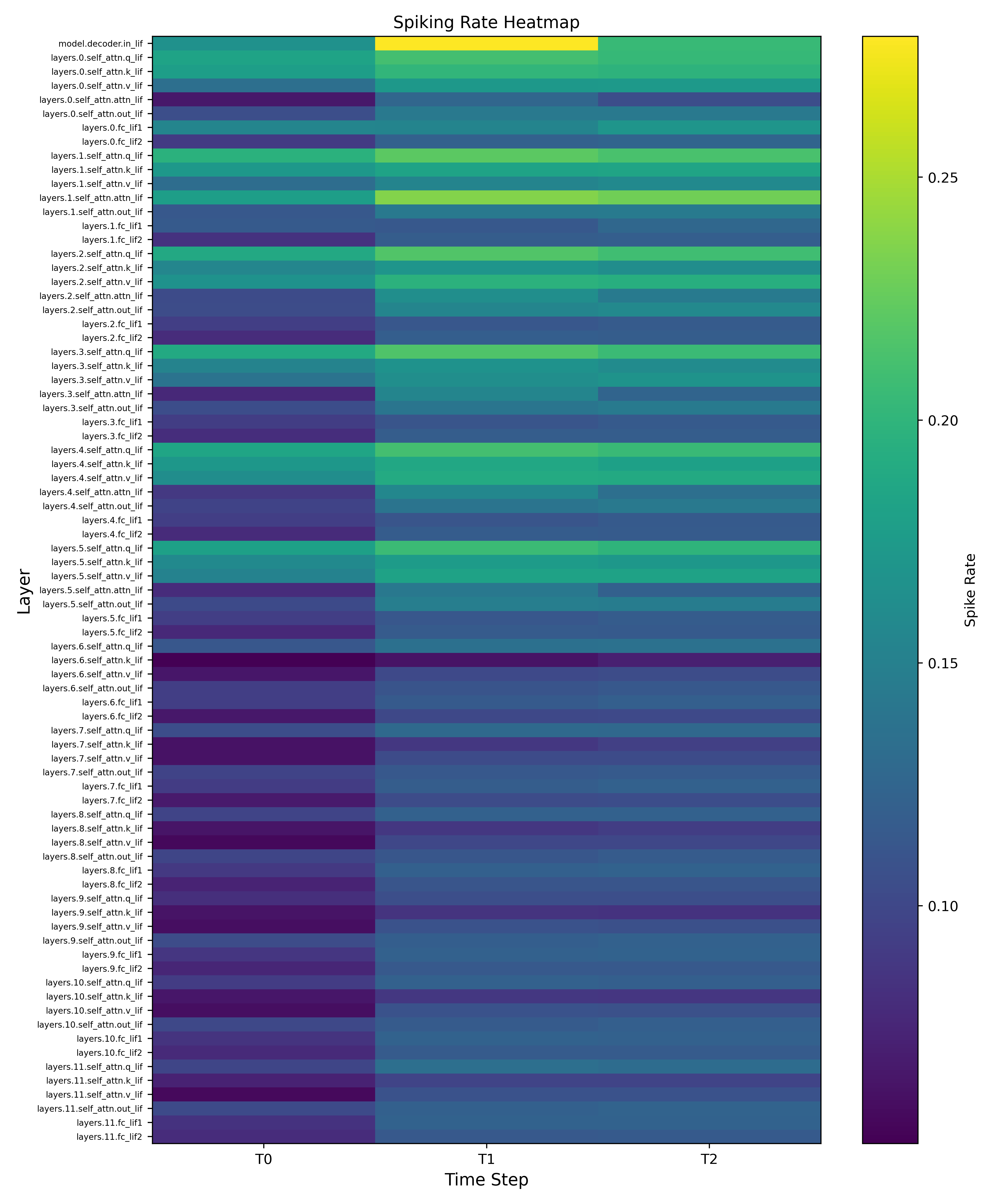}
        \caption{$T=3$}
        \label{fig:heatmap_t3}
    \end{subfigure}

    \vspace{0.3em} 

    \begin{subfigure}{0.42\textwidth}
        \centering
        \includegraphics[width=\linewidth]{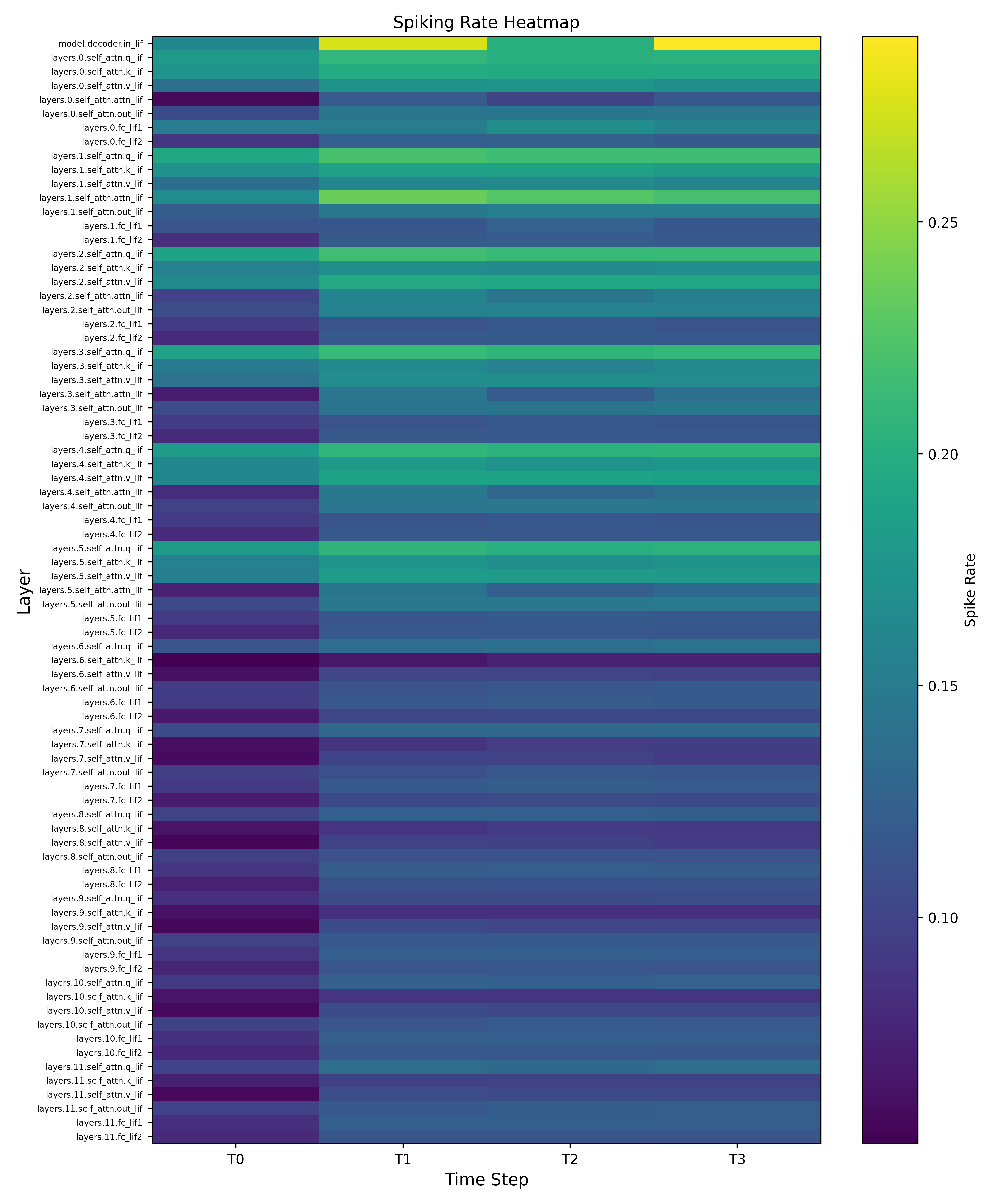}
        \caption{$T=4$}
        \label{fig:heatmap_t4}
    \end{subfigure}
    \hspace{0.02\textwidth}
    \begin{subfigure}{0.42\textwidth}
        \centering
        \includegraphics[width=\linewidth]{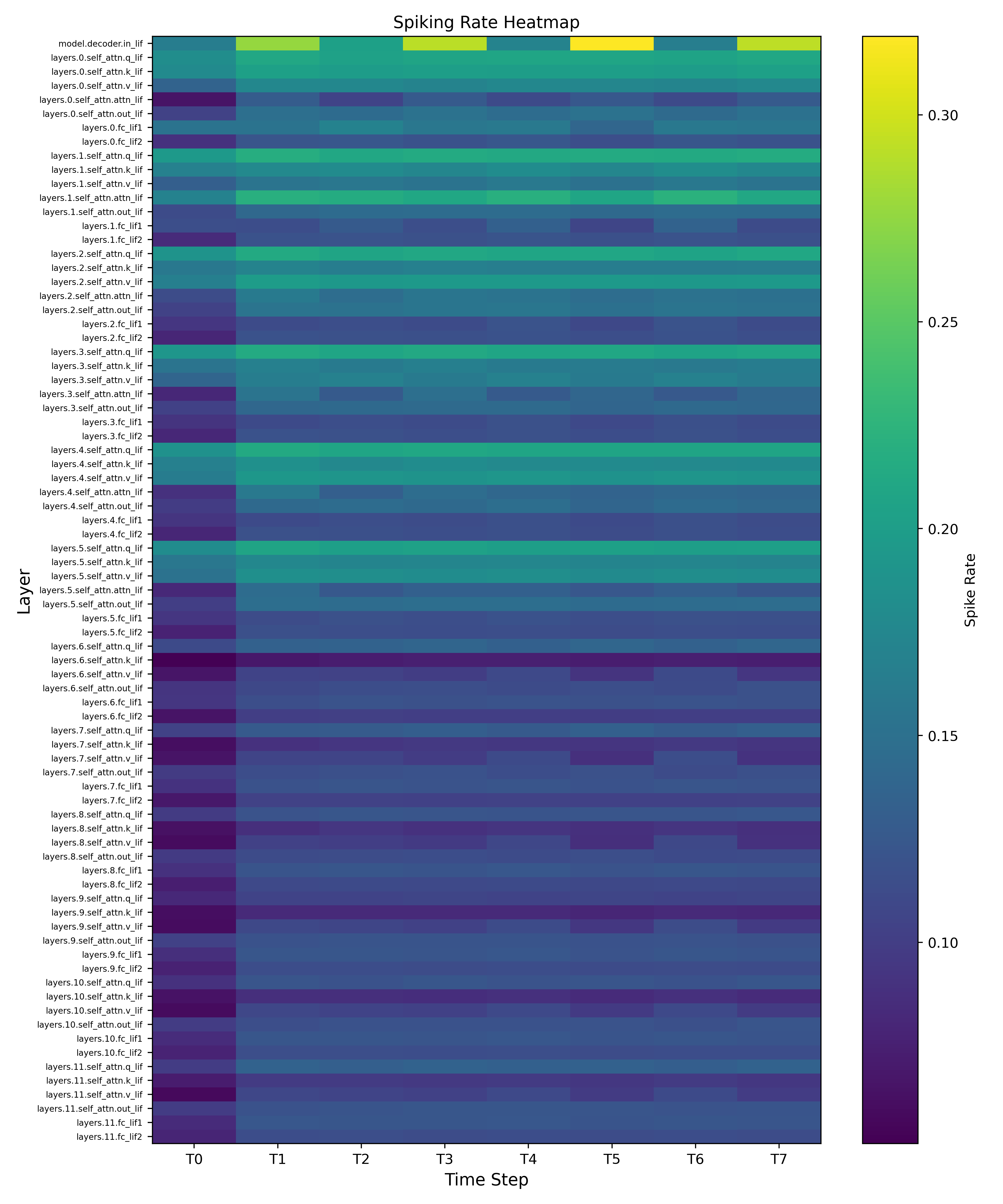}
        \caption{$T=8$}
        \label{fig:heatmap_t8}
    \end{subfigure}

    \caption{
Layer-wise firing rate heatmaps under different temporal resolutions.
(a) At $T=2$, firing activity is sparse and concentrated in a small subset of layers.
(b) At $T=3$, elevated firing rates begin to propagate across more layers, indicating increased temporal integration.
(c) With $T=4$, spiking activity becomes more distributed, engaging a larger portion of the network.
(d) At $T=8$, firing patterns are largely uniform across layers, suggesting a temporally saturated regime with broad layer participation.
    }
    \label{fig:heatmap_all}
\end{figure}

To investigate the temporal dynamics of spiking activity in \textbf{BiSpikCLM}, we visualize layer-wise firing rate distributions under different inference time steps, namely $T=2$, $T=3$, $T=4$, and $T=8$, as shown in Fig.~\ref{fig:heatmap_all}(a--d). In these heatmaps, the vertical axis corresponds to model layers, the horizontal axis represents discrete time steps, and the color intensity encodes the normalized firing rate from low (purple) to high (yellow).

As the number of time steps increases, the firing patterns exhibit a clear transition from sparse, layer-localized activation to more distributed and pervasive spiking activity. Under a limited temporal budget (e.g., $T=2$), high firing rates are concentrated in a small subset of layers, indicating that only essential components are activated for early-stage processing. With increased temporal resolution ($T=3$ and $T=4$), elevated firing activity gradually propagates across more layers, reflecting enhanced temporal integration and deeper hierarchical interactions. When the temporal budget is sufficiently large ($T=8$), the firing distribution becomes relatively uniform across most layers, suggesting a temporally saturated regime in which nearly all layers participate in the computation.

Overall, these observations indicate that \textbf{BiSpikCLM} dynamically adapts its internal computation to the available temporal budget, scaling from sparse and efficient activation to more expressive and distributed processing as additional time steps are provided. This temporal adaptability highlights the potential of spiking neural architectures for scalable and energy-aware language modeling.

\begin{figure}[H]
    \centering
    \includegraphics[width=0.8\textwidth]{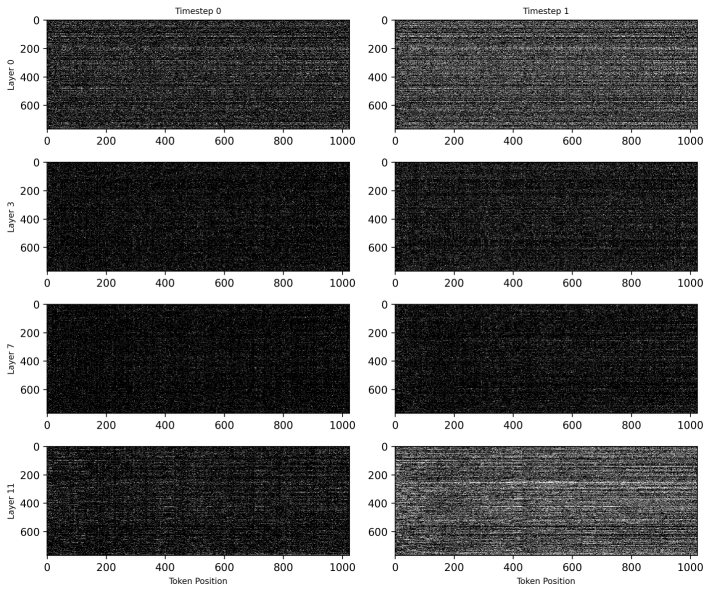}
    \caption{This figure presents the spiking activity across token positions for four representative layers (Layer 0, Layer 3, Layer 7, and Layer 11) at an early inference time step. The firing patterns are relatively sparse and uniformly distributed, particularly in the lower layers. This reflects the initial stage of neuronal processing, where the model begins encoding input signals with limited temporal context. Notably, deeper layers such as Layer 11 exhibit subdued activation, suggesting that higher-level abstractions have not yet emerged.}
    \label{fig:firing_2}
\end{figure}

While the layer-wise heatmaps reveal how spiking activity propagates across depth, the token-level visualizations in Figures~\ref{fig:firing_2}–\ref{fig:firing_8} further show how temporal integration shapes fine-grained, token-specific activation patterns. 

In addition to the layer-wise temporal spiking visualization, we further examine the firing patterns of \textbf{BiSpikCLM} at the level of individual token positions. Figures~\ref{fig:firing_2}–\ref{fig:firing_8} visualize token-wise spiking activity for selected layers (Layer 0, Layer 3, Layer 7, and Layer 11) under different inference time steps. In each figure, rows correspond to the selected layers, while columns represent discrete inference time steps. Within each subplot, the horizontal axis indicates token positions and the vertical axis indexes neurons (or channels) within the corresponding layer; color intensity encodes the firing magnitude, with lighter colors representing stronger activity.

\begin{figure}[H]
    \centering
    \includegraphics[width=0.8\textwidth]{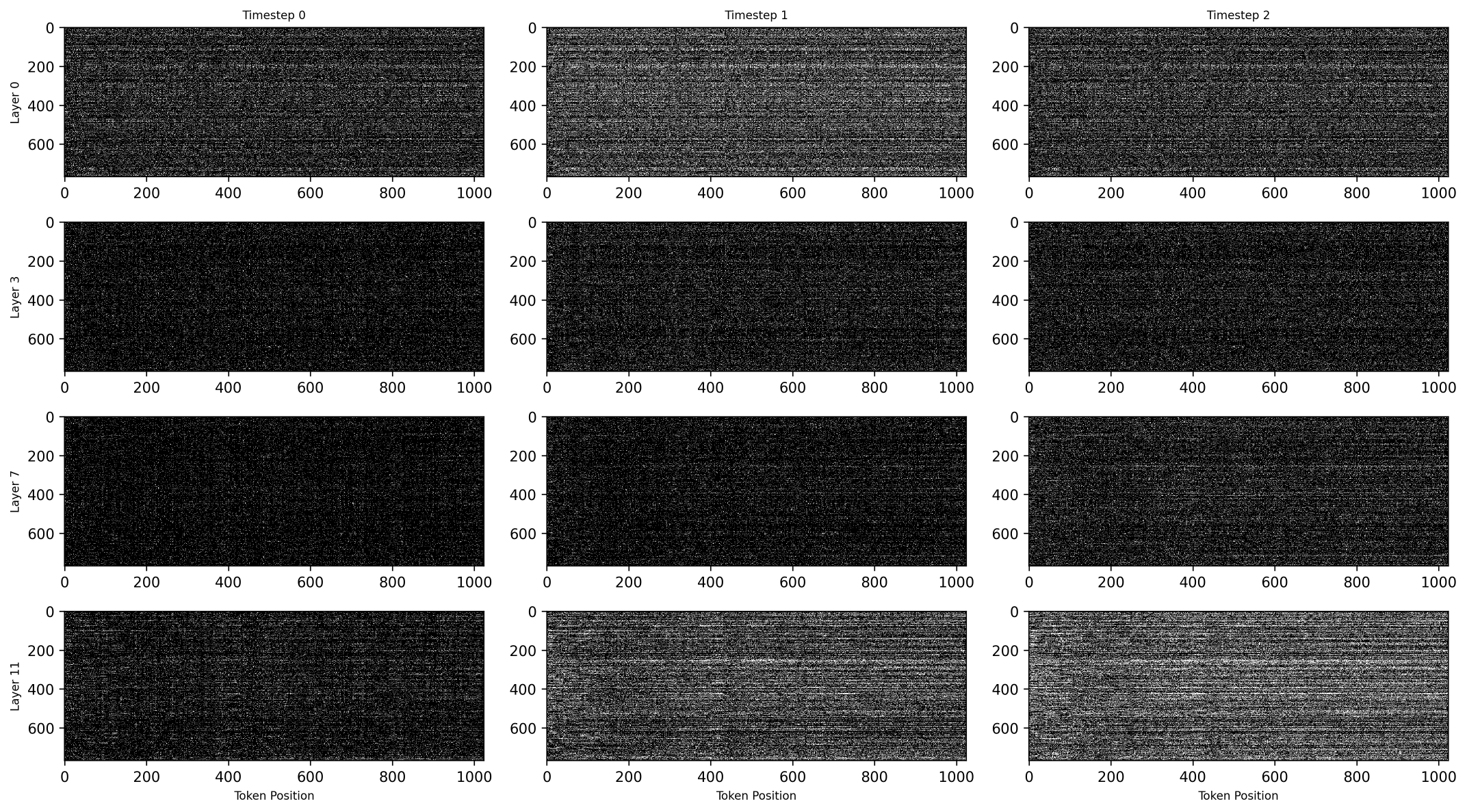}
    \caption{At $T=3$, the firing distributions become slightly more structured across token positions and layers. While lower layers maintain broadly distributed activity, deeper layers begin to display early signs of selective activation. Compared to $T=2$, this figure reveals the onset of temporal refinement, indicating that additional time steps allow the model to initiate more context-sensitive computation, particularly in the upper layers.}
    \label{fig:firing_3}
\end{figure}

\begin{figure}[H]
    \centering
    \includegraphics[width=0.8\textwidth]{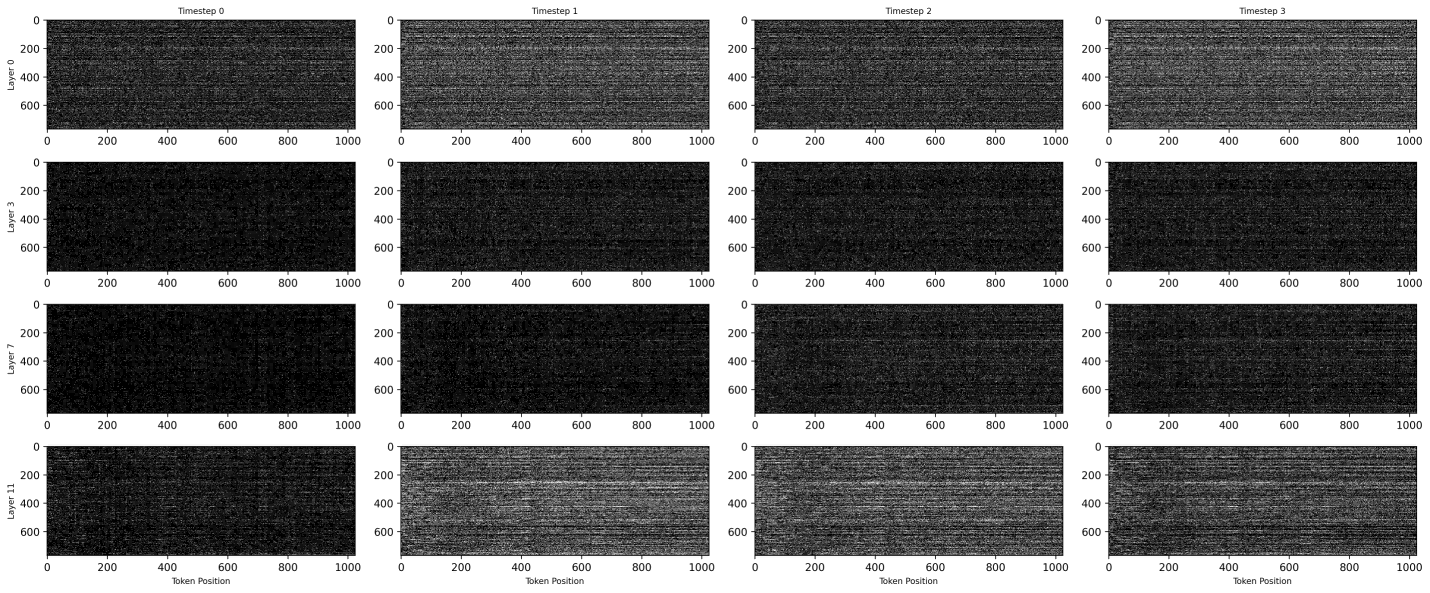}
    \caption{With four time steps, the model exhibits more pronounced spatiotemporal differentiation in firing behavior. Activity becomes more variable across token positions, and certain regions in deeper layers start to display concentrated firing. This suggests that the network is engaging in increasingly specialized processing, distributing its computation more selectively based on both input semantics and accumulated temporal evidence.}
    \label{fig:firing_4}
\end{figure}

These token-level visualizations provide a more fine-grained view of the internal computational dynamics of \textbf{BiSpikCLM}. At early inference stages (e.g., $T=2$), spiking activity remains relatively sparse and diffuse across token positions, particularly in intermediate layers, reflecting an initial encoding phase with limited temporal context. As additional time steps are introduced ($T=3$ and $T=4$), firing patterns become increasingly structured, with deeper layers exhibiting emerging selectivity over specific token regions, suggesting progressive temporal refinement and hierarchical processing. When the temporal budget is sufficiently large ($T=8$), deeper layers display pronounced and focused activation patterns, indicating refined internal representations and enhanced contextual integration.

Overall, these results demonstrate that \textbf{BiSpikCLM} dynamically allocates its computational resources across both layers and token positions as inference time increases, transitioning from broad and shallow activation to more selective and semantically rich processing. This behavior highlights the model’s capacity for temporally adaptive computation and underscores the suitability of spiking neural architectures for efficient and scalable language modeling.

\begin{figure}[H]
    \centering
    \includegraphics[width=0.8\textwidth]{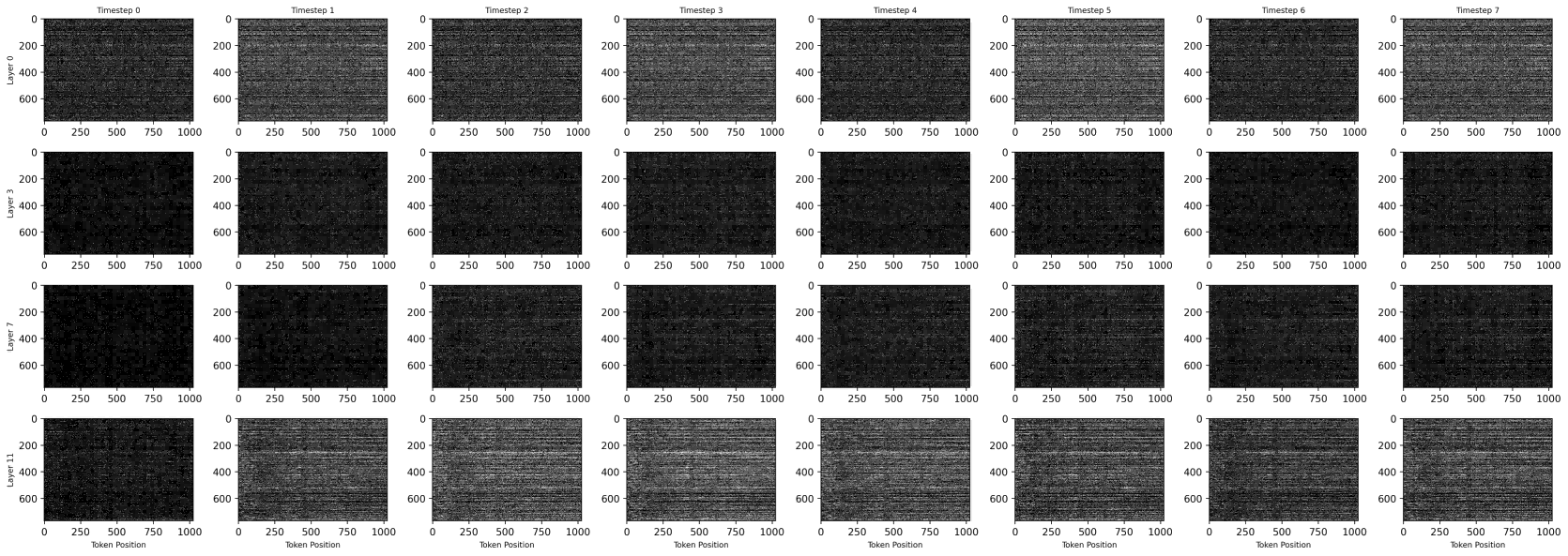}
    \caption{By $T=8$, the firing patterns exhibit substantial temporal evolution and structural complexity. Deeper layers, in particular, show heightened and more focused activation for specific token regions, reflecting refined internal representations. This level of activity suggests that the model has transitioned into a more stable and semantically rich encoding phase. The marked increase in firing diversity and intensity across layers highlights the model's capacity to utilize extended temporal windows for deeper contextual integration and task-specific computation.}
    \label{fig:firing_8}
\end{figure}

\section{Extending SFSA to Llama Architecture}

To demonstrate the versatility of our proposed SFSA (Softmax-Free Spiking Attention) mechanism, we extend its application beyond the original OPT-based BiSpikCLM architecture. Specifically, we implement SFSA on the Llama architecture~\citep{touvron2023llama}, training two new models: BiSpikCLM-Llama with 165M and 1.2B parameters. This expansion validates that SFSA is designed fundamentally for spike-based language modeling, independent of specific architectural choices.

\begin{table}[h]
\centering
\footnotesize
\setlength{\tabcolsep}{2.45pt} 
\renewcommand{\arraystretch}{1.1} 
\caption{Performance comparison of BiSpikCLM models across different architectures and time steps.}
\begin{tabular}{l|cc|ccccccccc}
\toprule
\multirow{2}{*}{\textbf{Model}} & 
\multirow{2}{*}{\textbf{\makecell{Params \\ (B)}}} & 
\multirow{2}{*}{\textbf{\makecell{Time \\ Step}}} & 
\multicolumn{9}{c}{\textbf{Zero-shot Accuracy (\%) $\uparrow$}} \\
& & & \textbf{ARC-e} & \textbf{ARC-c} & \textbf{WG} & \textbf{BQ} & \textbf{PIQA} & \textbf{HS} & \textbf{OBQA} & \textbf{HQA} & \textbf{Avg.} \\
\midrule
BiSpikCLM-OPT & 0.125 & 2 & 39.1 & 18.9 & 50.3 & 52.7 & 56.7 & 28.1 & 19.8 & 22.9 & 36.05 \\
BiSpikCLM-OPT & 0.125 & 4 & 39.4 & 19.0 & 51.2 & 53.0 & 57.5 & 29.2 & 19.7 & 23.1 & 36.50 \\
BiSpikCLM-Llama & 0.165 & 2 & 39.3 & 19.2 & 51.3 & 52.8 & 56.8 & 28.3 & 20.3 & 23.2 & 36.40 \\
BiSpikCLM-Llama & 0.165 & 4 & 39.6 & 19.6 & 51.7 & 53.2 & 57.3 & 29.5 & 20.7 & 23.7 & 36.91 \\
\midrule
BiSpikCLM-OPT & 1.300 & 2 & 45.7 & 23.5 & 54.2 & 56.3 & 62.3 & 40.2 & 24.5 & 24.0 & 41.33 \\
BiSpikCLM-OPT & 1.300 & 4 & 46.3 & 24.3 & 55.6 & 56.8 & 63.4 & 41.7 & 25.2 & 24.3 & 42.19 \\
BiSpikCLM-Llama & 1.200 & 2 & 45.8 & 23.9 & 54.7 & 56.1 & 63.0 & 40.6 & 24.7 & 24.1 & 41.61 \\
BiSpikCLM-Llama & 1.200 & 4 & 46.5 & 24.4 & 55.9 & 56.6 & 63.6 & 41.8 & 25.4 & 24.4 & 42.33 \\
\bottomrule
\end{tabular}
\end{table}

The results show consistent performance improvements when increasing time steps from T=2 to T=4 across all model variants. Notably, the Llama-based models achieve comparable or slightly better results than their OPT-based counterparts, particularly in the 1.2B parameter range where BiSpikCLM-Llama (T=4) reaches an average accuracy of 42.33\%. This demonstrates that SFSA effectively captures spiking dynamics across different transformer architectures while maintaining competitive language modeling capabilities.

\section{Comparison with Quantized ANNs}

To contextualize our contributions, it is crucial to distinguish Spiking Neural Networks (SNNs) from quantized Artificial Neural Networks (ANNs), a common point of comparison. Quantized ANNs achieve efficiency through \textbf{spatial discretization}, converting continuous floating-point weights or activations into low-bit, fixed-point formats (e.g., 2-bit, 4-bit). While this approach facilitates model compression and acceleration, the underlying computation remains fundamentally dependent on dense Multiply-Accumulate (MAC) operations. Although 1-bit networks eliminate MACs, they are notoriously difficult to train effectively.

In stark contrast, SNNs leverage \textbf{temporal sparse encoding} via binary spikes. A spike (1) or its absence (0) at a given time step encodes information, enabling event-driven and asynchronous computation. This paradigm allows for substantial energy savings on neuromorphic hardware by exploiting sparsity, a benefit difficult for quantized ANNs to replicate~\citep{horowitz2014energy}. Our work harnesses this inherent property of SNNs to significantly reduce inference energy while preserving language generation capabilities.

We provide a direct performance comparison between our 1.3B/1.2B BiSpikCLM models and state-of-the-art quantization methods in Table~\ref{tab:quant_comparison}. While quantized ANNs like \citet{shao2023omniquant}, \citet{wang2023bitnet}, and \citet{kaushal2024spectra} may exhibit a marginal edge in accuracy, we emphasize that these represent fundamentally different methodological paradigms.

Therefore, the primary contribution of our work is not to surpass quantization methods in accuracy, but to pioneer and validate a new, energy-efficient pathway for large language models. Our significance is threefold: (1) We introduce the first train-from-scratch binary-spike-based LLM; (2) We propose the Softmax-Free Spiking Attention (SFSA) mechanism to enable effective causal modeling in SNNs; and (3) We demonstrate the fundamental feasibility of SNN-LLMs, filling a critical gap in the field and establishing a foundation for future research into scalable, energy-conscious language models.

\begin{table}[h]
\centering
\footnotesize
\setlength{\tabcolsep}{2.45pt} 
\renewcommand{\arraystretch}{1.1} 
\caption{Performance comparison of BiSpikCLM with quantized LLMs.}
\label{tab:quant_comparison}
\begin{tabular}{l|c|ccccccccc}
\toprule
\multirow{2}{*}{\textbf{Model}} & 
\multirow{2}{*}{\textbf{\makecell{Params \\ (B)}}} & 
\multicolumn{9}{c}{\textbf{Zero-shot Accuracy (\%) $\uparrow$}} \\
& & \textbf{ARC-e} & \textbf{ARC-c} & \textbf{WG} & \textbf{BQ} & \textbf{PIQA} & \textbf{HS} & \textbf{OBQA} & \textbf{HQA} & \textbf{Avg.} \\
\midrule
BiSpikCLM (T=2) & 1.3 & 45.7 & 23.5 & 54.2 & 56.3 & 62.3 & 40.2 & 24.5 & 24.0 & 41.33 \\
BiSpikCLM (T=4) & 1.3 & 46.3 & 24.3 & 55.6 & 56.8 & 63.4 & 41.7 & 25.2 & 24.3 & 42.19 \\
BiSpikCLM-Llama (T=2) & 1.2 & 45.8 & 23.9 & 54.7 & 56.1 & 63.0 & 40.6 & 24.7 & 24.1 & 41.61 \\
BiSpikCLM-Llama (T=4) & 1.2 & 46.5 & 24.4 & 55.9 & 56.6 & 63.6 & 41.8 & 25.4 & 24.4 & 42.33 \\
\midrule
BitNet (1.58-bit) & 1.3 & 48.7 & 24.1 & 56.8 & 57.4 & 64.2 & 40.6 & 24.6 & 24.9 & 42.66 \\
SmoothQuant (W4A4) & 1.3 & 44.3 & 23.5 & 54.2 & 55.6 & 62.3 & 40.0 & 23.9 & 23.2 & 40.88 \\
OmniQuant (W4A4) & 1.3 & 49.7 & 25.5 & 58.1 & 58.3 & 66.0 & 41.7 & 26.2 & 25.7 & 43.90 \\
TriLM (1.58-bit) & 1.1 & 46.2 & 23.9 & 55.5 & 56.5 & 64.1 & 39.2 & 24.7 & 24.5 & 41.83 \\
TriLM (1.58-bit) & 1.5 & 49.1 & 25.2 & 57.3 & 57.2 & 64.5 & 41.1 & 25.6 & 25.3 & 43.16 \\
\bottomrule
\end{tabular}
\end{table}

\section{Effect of Training Scale and Conversational Ability}

To validate the scalability of our BiSpikCLM, we conducted experiments to assess the impact of increased training data. Our initial 1B and 10B token training runs were conducted under constrained GPU resources, primarily serving as a proof-of-concept. To further probe the potential of our model, we scaled the training for the 125M and 1.3B models to 5B and 25B tokens, respectively. While this scale is still modest compared to standard pre-training regimens, the use of knowledge distillation from a teacher model allows the student BiSpikCLM to learn more effectively, mitigating the extensive data requirements typically associated with training from scratch.

\begin{table}[h]
\centering
\footnotesize
\setlength{\tabcolsep}{3pt} 
\renewcommand{\arraystretch}{1.1} 
\caption{Performance of BiSpikCLM with varying training token counts.}
\label{tab:token_scaling}
\begin{tabular}{l|cc|ccccccccc}
\toprule
\multirow{2}{*}{\textbf{Model}} & 
\multirow{2}{*}{\textbf{\makecell{Params \\ (B)}}} & 
\multirow{2}{*}{\textbf{\makecell{Tokens \\ (B)}}} & 
\multicolumn{9}{c}{\textbf{Zero-shot Accuracy (\%) $\uparrow$}} \\
& & & \textbf{ARC-e} & \textbf{ARC-c} & \textbf{WG} & \textbf{BQ} & \textbf{PIQA} & \textbf{HS} & \textbf{OBQA} & \textbf{HQA} & \textbf{Avg.} \\
\midrule
BiSpikCLM (T=2) & 0.125 & 1.0 & 39.1 & 18.9 & 50.3 & 52.7 & 56.7 & 28.1 & 19.8 & 22.9 & 36.05 \\
BiSpikCLM (T=2) & 0.125 & 5.0 & 41.4 & 19.2 & 51.4 & 53.4 & 58.2 & 30.2 & 19.9 & 23.1 & 37.10 \\
\midrule
BiSpikCLM (T=2) & 1.300 & 10.0 & 45.7 & 23.5 & 54.2 & 56.3 & 62.3 & 40.2 & 24.5 & 24.0 & 41.33 \\
BiSpikCLM (T=2) & 1.300 & 25.0 & 48.3 & 26.4 & 57.8 & 58.6 & 65.1 & 44.9 & 27.3 & 26.7 & 44.39 \\
\bottomrule
\end{tabular}
\end{table}

As shown in Table~\ref{tab:token_scaling}, increasing the training tokens yields significant performance gains. The average accuracy of the 125M model improved by 1.05\% when trained on 5B tokens compared to 1B tokens. More impressively, the 1.3B model's average accuracy increased by 3.06\% when scaling from 10B to 25B tokens. These results demonstrate that BiSpikCLM is not merely a small-scale proof-of-concept but a model architecture that responds positively and effectively to increased training data, suggesting strong potential for further scaling.

To further investigate the training dynamics, we plot the training loss curves for our models. Figure~\ref{fig:loss_curve} illustrates the loss progression for the 1.3B model trained on both 10B and 25B tokens. The curves exhibit a smooth and consistent downward trend, indicating that our training process is stable and converges effectively. Notably, the model trained on 25B tokens continues to decrease its loss to a lower final value, corroborating the quantitative performance gains observed in Table~\ref{tab:token_scaling}. This stable convergence behavior across different training scales demonstrates the robustness of our proposed BiSpikCLM architecture and the effectiveness of the SFSA mechanism in facilitating the optimization of spike-based language models.

\begin{figure}[h]
\centering
\includegraphics[width=0.9\textwidth]{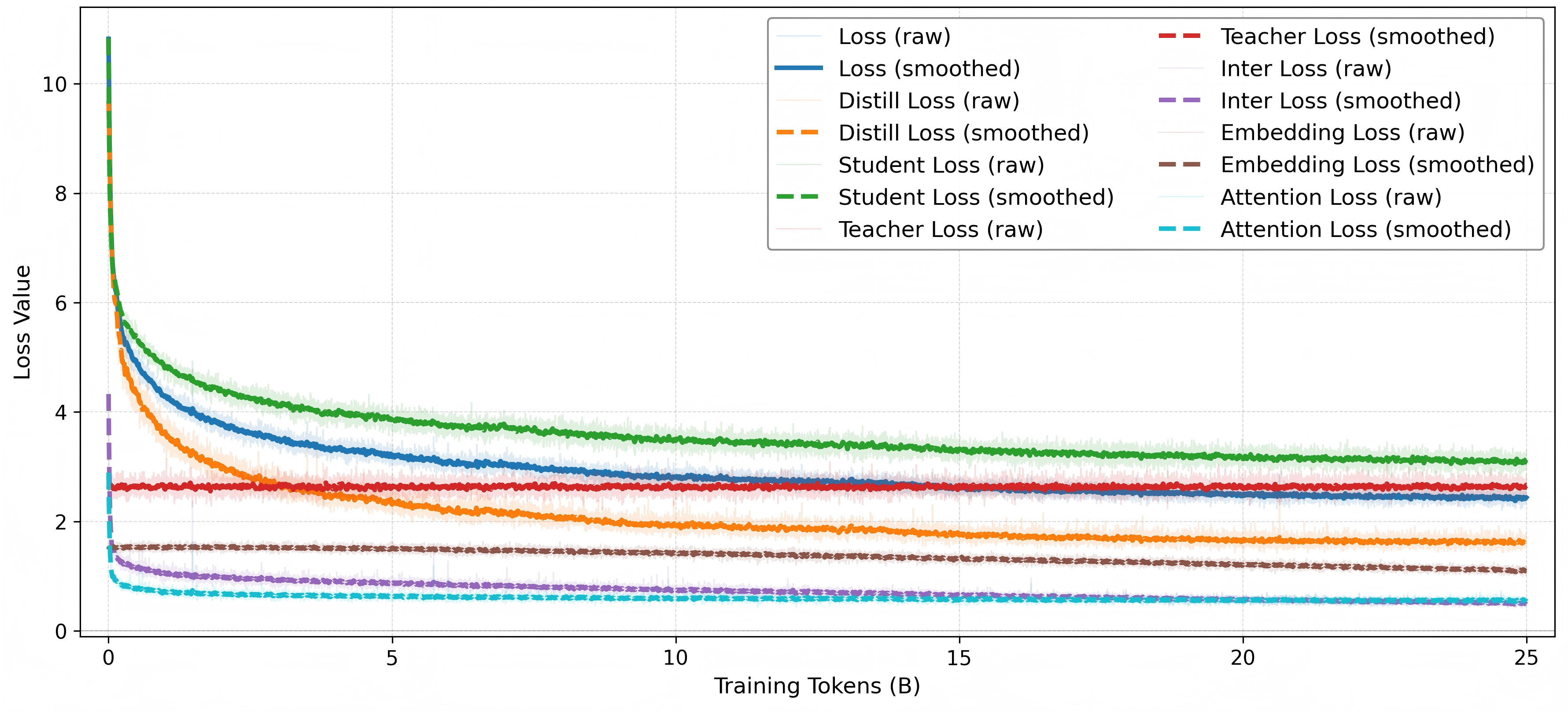}
\caption{Training loss curves for the BiSpikCLM-1.3B model, comparing the 10B and 25B token training regimes. The smooth downward trend confirms stable convergence.}
\label{fig:loss_curve}
\end{figure}

Beyond quantitative metrics, we qualitatively evaluated the conversational abilities of our BiSpikCLM-1.3B model trained with 10B and 25B tokens. Table~\ref{tab:conversation_cases} presents sample responses to a set of prompts. The 10B model exhibits basic conversational skills and can maintain a simple dialogue. However, it fails to answer factual questions correctly, likely due to a lack of exposure to that specific knowledge during its limited training. In contrast, the 25B model provides more accurate and detailed responses, correctly answering the factual question about the U.S. capital and generating more coherent and contextually relevant text. This qualitative improvement aligns with the quantitative gains, underscoring the importance of training scale for developing the factual knowledge and reasoning capabilities of spiking language models.

\begin{table}[h]
\centering
\footnotesize
\caption{Comparison of conversational abilities for BiSpikCLM-1.3B trained on different token counts.}
\label{tab:conversation_cases}
\begin{tabular}{p{0.2\textwidth}p{0.35\textwidth}p{0.35\textwidth}}
\toprule
\textbf{Prompt} & \textbf{BiSpikCLM (10B Tokens)} & \textbf{BiSpikCLM (25B Tokens)} \\
\midrule
\textit{Hi there! How are you doing today?} & 
Hi! How are you? I'm doing well. How about you? & 
I'm doing great! I'm just tired as hell. I went to bed at 11pm last night, and now I'm up at 4am. I have a class tomorrow morning at 7:30, so I have to get up and go to class. \\
\addlinespace
\midrule
\textit{Where is the capital of the United States?} & 
It's in the middle of the country. & 
Washington, D.C. \\
\addlinespace
\midrule
\textit{How about the cat?} & 
I think it's a good idea to have a cat in the house. & 
I have a cat and he doesn't like the camera. I'll send you a picture when he is in the mood. \\
\bottomrule
\end{tabular}
\end{table}

\section{Analysis of Long-Context Dependency and Sparsity}

In the Softmax-Free Spiking Attention (SFSA) mechanism, information propagation relies on spike-triggered updates of the attention state. Consequently, a higher sparsity level—achieved by increasing the firing threshold or reducing the number of time steps—can theoretically weaken the transmission of long-range signals, potentially impairing the model's ability to model long-distance dependencies.

To empirically evaluate this, we conducted a sliding-window evaluation on the WikiText-2 dataset. We compared our BiSpikCLM-1.3B (T=4) and BiSpikCLM-Llama-1.2B (T=4) models with varying firing thresholds against the full-precision OPT-1.3B and Llama-3.2-1.2B baselines. The models were tested with context lengths of 512, 1024, 2048, 4096, and 8192 tokens, and the performance was measured using Perplexity (PPL). For reference, the maximum supported context length for OPT-1.3B is 2048. The results are presented in Table~\ref{tab:long_context_ppl}.

\begin{table}[h]
\centering
\footnotesize
\setlength{\tabcolsep}{3pt}
\renewcommand{\arraystretch}{1.1}
\caption{Perplexity (PPL) on WikiText-2 with varying context lengths and firing thresholds. The firing rate (sparsity) is shown in parentheses.}
\label{tab:long_context_ppl}
\begin{tabular}{l|cc|ccccc}
\toprule
\multirow{2}{*}{\textbf{Model}} & 
\multirow{2}{*}{\textbf{\makecell{Params \\ (B)}}} & 
\multirow{2}{*}{\textbf{\makecell{Firing \\ Threshold}}} & 
\multicolumn{5}{c}{\textbf{Context Length}} \\
& & & \textbf{512} & \textbf{1024} & \textbf{2048} & \textbf{4096} & \textbf{8192} \\
\midrule
OPT & 1.300 & - & 16.26 & 13.58 & 11.13 & - & - \\
Llama-3.2 & 1.200 & - & 12.93 & 10.96 & 9.76 & 9.02 & 8.54 \\
\midrule
\multirow{3}{*}{BiSpikCLM} 
& \multirow{3}{*}{1.300} & 0.70 & 36.72 (0.184) & 32.49 (0.183) & 29.34 (0.181) & - & - \\
& & 0.85 & 39.75 (0.177) & 36.30 (0.175) & 32.18 (0.174) & - & - \\
& & 1.00 & 43.17 (0.163) & 39.88 (0.162) & 34.62 (0.160) & - & - \\
\midrule
\multirow{3}{*}{BiSpikCLM-Llama} 
& \multirow{3}{*}{1.200} & 0.70 & 33.53 (0.198) & 29.82 (0.195) & 27.59 (0.194) & 25.32 (0.192) & 23.77 (0.190) \\
& & 0.85 & 37.14 (0.186) & 33.37 (0.185) & 31.26 (0.185) & 28.89 (0.183) & 25.42 (0.181) \\
& & 1.00 & 40.88 (0.174) & 37.25 (0.172) & 34.08 (0.171) & 30.76 (0.171) & 27.19 (0.168) \\
\bottomrule
\end{tabular}
\end{table}

The results lead to several key observations. First, as the sparsity increases (i.e., the threshold rises from 0.70 to 1.00), the PPL performance degrades across all context lengths. This confirms our intuition that higher sparsity can impede the flow of information. However, the magnitude of this degradation is moderate, suggesting that increasing sparsity within a certain range does not cause a catastrophic collapse in long-context dependency handling.

And more intriguingly, for a fixed threshold, the PPL consistently decreases as the context length increases. For instance, the BiSpikCLM-Llama model with a threshold of 0.70 improves from a PPL of 33.53 at 512 tokens to 23.77 at 8192 tokens. This indicates that the natural increase in sparsity caused by longer sequences has a negligible negative impact on the model's capability. The model retains its ability to extract global information from longer contexts.

Finally, we acknowledge that the current design of SFSA lacks specialized mechanisms for handling extremely long contexts. We are actively working to address this limitation through future research directions, such as adaptive spike scheduling and the integration of long-term memory neurons. We believe these are engineering frontier challenges rather than fundamental obstacles and can be progressively overcome in future work.


\end{document}